\documentclass[12pt]{article}
\usepackage{amsfonts, amsmath, amssymb, amsthm}
\usepackage{mathrsfs}
\usepackage{xcolor}
\usepackage[hidelinks]{hyperref}
\usepackage[margin = 2.5cm]{geometry}
\usepackage{graphicx}
\usepackage{caption}
\usepackage{float}
\usepackage{adjustbox}
\usepackage[round]{natbib}
\usepackage{multirow,multicol}
\usepackage{comment}
\newtheorem{theorem}{Theorem}

\newtheorem{lemma}{Lemma}
\newtheorem{remark}{Remark}
\newtheorem{assumption}{Assumption}
\newtheorem{corollary}{Corollary}
\usepackage{subcaption}\usepackage[font=small,labelfont=small]{caption}
\usepackage{enumerate}
\usepackage{xurl}
\usepackage{hyperref}
\usepackage{cleveref} 
\hypersetup{
    colorlinks=true,
    linkcolor=red,      
    citecolor=blue,    
    urlcolor=blue       
}

\begin{document}

\title{Nonparametric estimation of conditional probability distributions using a generative approach based on conditional push-forward neural networks}

\author{Nicola Rares Franco$^1$, Lorenzo Tedesco$^2$}
\date{\small $^1$MOX, Department of Mathematics, Politecnico di Milano\\
$^2$Department of Economics, University of Bergamo}

\maketitle

\newcommand{\Xspace}{\mathcal{X}}
\newcommand{\Yspace}{\mathcal{Y}}
\newcommand{\XYspace}{\mathcal{S}}
\newcommand{\Aspace}{\mathcal{A}}
\newcommand{\hypothesis}{\mathscr{H}^\varepsilon}
\newcommand{\kpdf}{\kappa}
\newcommand{\ourmethodd}{CPFN}
\newcommand{\ourmethod}{\ourmethodd\,}
\newcommand{\nicolahere}{\begin{center}\textcolor{red}{\textnormal{--- Nicola: sono arrivato qui. ---}}\end{center}}
\newcommand{\review}[1]{\textcolor{black}{#1}}

\begin{abstract}
We introduce \emph{conditional push-forward neural networks} (CPFN), a generative framework for conditional distribution estimation. Instead of directly modeling the conditional density $f_{Y|X}$, CPFN learns a stochastic map $\varphi=\varphi(x,u)$ such that $\varphi(x,U)$ and $Y|X=x$ follow approximately the same law, with $U$ a suitable random vector of pre-defined latent variables. This enables efficient conditional sampling and straightforward estimation of conditional statistics through Monte Carlo methods. The model is trained via an objective function derived from a Kullback–Leibler formulation, without requiring invertibility or adversarial training. We establish a near-asymptotic consistency result and demonstrate experimentally that CPFN can achieve performance competitive with—or even superior to—state-of-the-art methods, including kernel estimators, tree-based algorithms, and popular deep learning techniques, all while remaining lightweight and easy to train.
\end{abstract}
\section{Introduction}


Estimating the conditional distribution of a response variable $Y\in\mathbb{R}^q$ given a vector of covariates $X\in\mathbb{R}^p$ is a timely problem of general interest. 
In fact, conditional probabilities are the finest tool by which one can describe the influence of a variable over another, providing a much richer perspective compared to, e.g., least squares regression. 
Indeed, assuming finite variances, it is well-known that the best approximation of $Y$ as a deterministic function of $X$ is given by the conditional expectation $\mathbb{E}[Y\,|\,X]$, in that $ \mathbb{E}[Y\,|\,X]=\Phi_*(X)$ for a suitable $\Phi_*:\mathbb{R}^d\to\mathbb{R}^q$, and
\begin{equation}
    \label{eq:least-squares}
    \mathbb{E}\|Y - \Phi_*(X)\|^2 =\inf_{\Phi}\mathbb{E}\|Y-\Phi(X)\|^2,
\end{equation}
the infimum being over all measurable maps $\Phi$ from $\mathbb{R}^d$ to $\mathbb{R}^q$. This representation, however, is very limited, as it does not capture the variability of $Y$ given $X$. Specifically, aside from the mean, it provides nearly no information about the distribution of $Y\,|\,X=x$, such as multimodality, skewness, variance, etc., whose description becomes essential whenever $Y$ is not univocally determined by $X$. In general, all of these aspects disappear completely in the least squares analysis, and they can only be recovered via the estimation of conditional probabilities \citep{hothorn2014conditional}.

In the literature there are several approaches for the estimation of conditional distributions, which can be subdivided into \emph{parametric} and \emph{nonparametric} ones. The first assume that $X$ and $Y$ are linked by a specific relationship ---e.g. linear---, or that they belong to a special class of random variables ---e.g., Gaussians. Then, the probability law of $Y\,|\,X$ is fully determined. Nonparametric techniques, instead, refrain from introducing highly restrictive assumptions, and favor representational richness over interpretability. Here, we shall focus our attention over this second class of approaches, specifically addressing the case of absolutely continuous random variables.

\paragraph{Literature review}
Historically, some of the first nonparametric approaches to conditional density estimation are those falling in the family of \emph{kernel conditional density estimation} methods (KCDE), see, e.g., \cite{fan1996,hyndman1996, holmes2010fast}. Closely related to Nadaraya–Watson estimators \citep{nadaraya1964estimating, watson1964smooth}, KCDE techniques leverage the fact that the conditional density of $Y\,|\,X=x$ can be expressed in terms of the joint density, $f_{YX}$, and the marginal density, $f_{X}$, via the well-known formula
\begin{equation}
    \label{eq:conditional ratio}
    f_{Y|X}(y\,|\,x)=\frac{f_{XY}(x,y)}{f_X(x)}.
\end{equation}
By integrating the machinery of kernel methods in the above, KCDE provide close form solutions to the problem of conditional density estimation. However, KCDE techniques also have their shortcomings. For instance, KCDE can encounter significant challenges when $d+q$ is fairly large, which limits its applicability in high-dimensional settings. Moreover, the implementation of KCDE requires setting several hyperparameters, such as the \emph{bandwidth}, whose selection is often challenging and continues to be an active topic of research: see, e.g., \cite{bashtannyk2001bandwidth, konevcna2019maximum, zhao2025adaptive}.


Later improvements with respect to KCDE, such as local likelihood estimation \citep{hyndman2002} and nearest-neighbor kernel estimators \citep{hall2004} mitigate some bias–variance trade-offs but remain limited in high-dimensional settings. Alternative approaches, including least-squares conditional density estimation \citep{sugiyama2010conditional} and direct conditional density ratio estimation \citep{kanamori2012}, bypass explicit density estimation by optimizing a surrogate risk for the conditional density ratio, achieving better scalability but still relying on Gaussian basis expansions.
\\\\
In recent years, in opposition to kernel methods, tree-based and histogram-based estimators have gained renewed attention for their computational advantages. Early works, such as \cite{cousins2019cadet}, adapted decision trees to conditional density estimation using Gaussian-leaf models, but their unimodal assumption limited flexibility. More recently, \cite{yang2024conditional} proposed a fully nonparametric tree-based approach using histogram-valued leaves, optimizing splits through a minimum description length criterion. Ensemble methods, such as \cite{pospisil2018rfcde}, combine random forests with kernel smoothing to obtain flexible yet structured estimates.

However, while these methods can be very attractive for their simplicity and flexibility, they can also be very sensitive to discretization and binning choices, which can pose significant challenges, particularly in high-dimensional or data-limited settings. Furthermore, their effective interpretability, often attributed to feature importance rankings or partial dependence plots, remains a subject of debate. In fact, tree-based approaches tend to be sensitive to variations in the training data or hyperparameter choices, which may influence derived interpretations \citep{huang2016parameter}. 


\;\\
An alternative perspective is brought by transformation 
methods. 
By leveraging the principles underlying the inverse transform method \citep{devroye2006nonuniform},  \emph{conditional transformation models} \citep{hothorn2014conditional, shen2020conditional, baumann2021deep} focus on approximating the conditional cumulative distribution function (CDF) rather than the conditional density. To this end, they parametrize the conditional CDF in terms of suitable $x$-dependent transformations. 
These methods, however, remain limited to 
univariate responses variables.

A similar idea stands at the core of \emph{conditional normalizing flows} \citep{winkler2019learning}, a deep learning framework that leverages $x$-dependent invertible transformations to map arbitrary conditional densities onto Gaussian distributions and viceversa. Although entirely black-box, conditional normalizing flows have proven remarkably accurate, with state-of-the-art performances across many different tasks. 
However, they also entail some limitations. In fact, the architectural constraints imposed by the invertibility assumption can significantly limit their design, and naive implementations can result in elevated computational costs. Simpler, more scalable, variants of normalizing flows exists, see e.g. \cite{dinh2015nice} and \cite{dinh2017density},  but their expressiveness is often limited, meaning that one may need very deep architectures to represent complicated distributions \citep{kobyzev2020normalizing}.

It is worth mentioning that normalizing flows share interesting connections with other approaches in the literature as well. For instance, their training process entails minimizing a negative log-likelihood functional—an idea already present in early works such as mixture density networks \citep{bishop1994}.
Another important aspect is that conditional normalizing flows can be regarded as \emph{generative models}, a broad family encompassing various architectures, including conditional variational autoencoders \citep{sohn2015learning} and generative adversarial neural networks \citep{goodfellow2014generative, song2025wasserstein}. We argue that these two components ---negative log-likelihood minimization and generative capability--- are key, reason for which we shall include them in our newly proposed approach.

The first one, in fact, bridges modern machine learning techniques with the more traditional and theoretically grounded field of statistical learning, thereby providing a solid conceptual foundation. The second, instead, guarantees additional flexibility at inference time. In fact, even when one has access to an estimated conditional density—as in kernel-based methods—computing such quantities remains nontrivial, as it typically requires implementing appropriate quadrature rules for numerical integration. In contrast, generative models, while not yielding explicit conditional densities, allow direct sampling from the conditional distribution, making these estimations readily achievable through classical Monte Carlo methods—especially when the sampling procedure is computationally efficient.


\paragraph{Our contribution}
Building upon these considerations, we propose a novel nonparametric approach for the estimation of conditional distributions based on a generative paradigm. 
Thus, instead of approximating $f_{Y|X}$, our approach aims at constructing a generative model that can efficiently sample from the conditional law of ${Y\,|\,X}$. This allows for a comprehensive description of the relationship between the two variables, while simultaneously enabling the estimation of conditional statistics in the natural way ---including, in the univariate case, the computational of conditional quantiles of arbitrary levels.

To this end, our construction makes use of deep neural networks. However, differently from other existing techniques, our approach does not require introducing  auxiliary architectures ---as in adversarial methods---, nor does it impose strict constraints in the design phase ---as in normalizing flows. 
Furthermore, unlike most generative approaches—including recent scoring-rule–based methods \citep{pacchiardi2024probabilistic, bulte2025probabilistic}—our framework makes the conditional density explicit, recovering the familiar maximum-likelihood formulation and thus promoting interpretability and theoretical insight.
Our experiments show that the proposed approach can achieve state-of-the-art performances while remaining lightweight enough to be trained on standard laptop hardware.


\paragraph{A note on quantile regression}
The framework discussed in this work can also be connected to the rich literature on \emph{quantile regression}. The latter, in fact, aims to estimate the conditional quantile functions of a response variable $Y$ given predictors $X$. Specifically, when the response variable $Y$ is univariate (i.e., $d = 1$), given a quantile level $\tau \in (0,1)$, quantile regression seeks to estimate the value $q_{Y|X}(\tau)$ as a function of $X$ such that
$$
\mathbb{P}(Y \leq q_{Y|X}(\tau) \mid X) = \tau, 
$$
Quantile regression methods can also be divided into parametric and nonparametric ones. In particular, a classic parametric approach is to model
the relationship between
$X$ and $q_{Y|X}(\tau)$ as linear, an approach known as linear quantile regression. In contrast, nonparametric quantile regression, 
leverages kernel methods, ensemble-based models ---as in quantile regression forest---, and spline functions, allowing for smooth, data-adaptive approximations of the conditional quantile functions. These methods are  particularly useful in settings characterized by heteroskedasticity, skewness, or heavy-tailed noise, where mean-based models may fail to provide a complete picture \citep{koenker2005quantile}.
Spline methods, in particular, can efficiently represent intricate relationships while maintaining interpretability and smoothness ---see, e.g., \cite{meinshausen2006quantile, koenker1994quantile}. In general, all such techniques typically involve the minimization of a \emph{pinball loss} ---also known as quantile loss---, which serves to directly estimate the desired quantile levels.

For the more general purpose of estimating conditional probabilities, however, traditional quantile regression methods entail some inherent limitations. In fact, most approaches require  fixing the quantile level $\tau$ before fitting, which hinders a comprehensive exploration of the conditional distribution. Furthermore, simply fitting multiple models for varying $\tau$ does not completely solve the problem, as it fails to guarantee essential properties such as the monotonicity of the quantile map $\tau \mapsto q_{Y|X}(\tau)$ —thus motivating the integration of suitable \emph{ad hoc} corrections; see, e.g., \cite{wang2024generative}.

Additionally, extending quantile regression to the multivariate case ($d > 1$) is challenging, primarily because there is no natural ordering in higher dimensions, which rises multiple competing definitions of multivariate quantiles (e.g., component-wise, spatial, or directional, among others). On top of this, the non-differentiability of the pinball loss at target values and the complexity of modeling dependencies among outputs can hinder optimization, especially in high-dimensional settings. These factors significantly hinder optimization and make multivariate quantile regression both theoretically and practically difficult \citep{serfling2002quantile}.

\paragraph{Outline of the paper}
The paper is organized as follows. First, in Section~\ref{sec:method}, we presented our proposed approach, describing its implementation and underlying rationale. Then, in Section~\ref{sec:theory}, we derive some theoretical results supporting our construction, primarily proving a weak form of consistency. Things are then put into action in Sections~\ref{sec:simulation}-\ref{sec:test}, where we assess the performances of the proposed approach over simulated and real data, respectively. Finally, Section~\ref{sec:conclusions} offers concluding remarks and discusses potential directions for future research.

\review{
\begin{remark}
It has come to our attention that \cite{kumar2024likelihood} recently explored similar ideas when analyzing the theoretical properties of classical generative models, such as variational autoencoders. The authors, however, refrained from using their theoretical construction in the actual experiments, sticking to variational approximations of the negative loglikelihood ---such as the Evidence Lower Bound (ELBO)--- when training their generative models. In this sense, our proposed approach can be seen as the practical counterpart of an otherwise purely theoretical construction. Consequently, we find that a brief comparison with the work of Kumar et al. can be helpful in placing our contribution in context. 
For instance, both works discuss convergence results in the infinite sample size limit, addressing theoretical properties such as near-asymptotic consistency. However, while \cite{kumar2024likelihood} consider generative models affected by Gaussian noise, our analysis focuses on smooth perturbations with compact support. For this reason, during our theoretical exploration, we end up making slightly different assumptions and arrive at slightly different conclusions.\\\indent
As we anticipated, another notable difference lies in the practical implementation of the approach. Our construction trains the model by effectively minimizing the negative log-likelihood of the data, without the need to introduce variational approximations. 
This allows us to implement the generative model using a single neural network architecture. Conversely, by relying on an ELBO approximation, \cite{kumar2024likelihood} effectively replicate the architecture of a variational autoencoder, whose training requires the introduction of an auxiliary encoder network. Additionally, \cite{kumar2024likelihood} propose the use of dense and sparse networks for the generative model, whereas we rely on a low-rank architecture decoupling the conditioning variables and the latent stochasticity. Finally, while \cite{kumar2024likelihood} extend the theory by examining singular conditional distributions supported on low-dimensional manifolds, our work aims at providing a comprehensive analysis over real data, benchmarking the proposed approach against state-of-the-art techniques beyond CKDE, such as tree-based methods and normalizing flows.
\end{remark}
}

\section{Presentation of the proposed approach}
\label{sec:method}
We devote this section to the presentation of our proposed approach, describing its main constituents and discussing the underlying idea. In what follows, we assume the reader to be familiar with standard deep feedforward neural networks, see, e.g., Chapter VI of \cite{goodfellow2016deep}. 
We recall that the latter are maps $\varphi:\mathbb{R}^{a_0}\to\mathbb{R}^{a_\ell}$ obtained via the composition of multiple layers,
$$\varphi=L_\ell\circ\dots\circ L_1,$$
with each layer $L_k$ consisting of an affine transformation $T_k:\mathbb{R}^{a_{k-1}}\to\mathbb{R}^{a_k}$ and a (scalar) nonlinearity $\rho_k:\mathbb{R}\to\mathbb{R}$,
$$L_k(v)=\rho_k(T_kv).$$
Here, as custom in the machine learning community, the action of the nonlinearity over a vector input is intended componentwise. The layers $L_1,\dots L_{\ell-1}$ are referred to as \emph{hidden layers}, and they typically share the same activation function, $\rho_1=\dots=\rho_{\ell-1}$. Conversely, $L_\ell$ is known as the \emph{output layer}. Depending on the task, the output layer may or may not include a nonlinear activation function. 
The shift and the linear operator defining the affine transformations $T_k$, instead,  are typically represented with a so-called bias vector, $b_k\in\mathbb{R}^{a_k}$, and a weight matrix, $W_k\in\mathbb{R}^{a_k\times a_{k-1}}$, respectively.

\subsection{Conditional push-forward neural networks: general idea}
In order to  sample from the conditional distribution of $Y\,|\,X$, we shall construct a neural network architecture 
$$\varphi:\mathbb{R}^d\times\mathbb{R}^q\to\mathbb{R}^q,$$
accepting two inputs, $\varphi=\varphi(x,u)$. The first one corresponds to the value attained by $X$, whereas the second one is random and serves to model the stochasticity in $Y$ that cannot be explained through $X$ alone. Precisely, the idea is to construct $\varphi$ such that for all $x$,
$$Y\,|\,X=x\quad\quad\text{and}\quad\quad\varphi(x,U)$$
follow approximately the same probability distribution, where $U\in\mathbb{R}^q$ is a suitable random vector, independent of $X$ and $Y$. Ideally, $U$ should be designed so that different realizations can be sampled easily. For instance, $U\sim\mathcal{U}[0,1]^q$ could be either uniform or $U\sim\mathcal{N}(0,\mathbf{I}_q)$ standard Gaussian.

In other words, denoting by $\mathbb{P}_{Y|X=x}$ the conditional law of $Y\,|\,X$ and by $\mathbb{Q}$ the probability distribution of $U$, our proposed approach consists in constructing a neural network $\varphi$ such that
$$\mathbb{P}_{Y|X=x}\approx\varphi(x,\cdot)\,\sharp\,\mathbb{Q},$$
where $\sharp$ denotes the push-forward operator. For this reason, we shall call such $\varphi$ a \textit{conditional push-forward neural network} (CPFN). 


The idea is inspired by basic results in the theory of Optimal Transport \citep{villani2008optimal} concerning \emph{couplings}: see, e.g., the Knothe--Rosenblatt rearrangement and measurable isomorphisms. More precisely, for each $x \in \mathbb{R}^d$, these results ensure that if $Y \mid X=x$ and $U$ are absolutely continuous, then there exists a transport map $\varphi_x : \mathbb{R}^d \to \mathbb{R}^d$ such that $\varphi_x(U)$ has the same distribution as $Y \mid X=x$. Further arguments show that in most cases, as with Knothe--Rosenblatt rearrangements, the transport map can be chosen to depend measurably on $x$, which naturally motivates the reformulation $\varphi_x = \varphi(x, \cdot)$. 
Thus, conditional push-forward neural networks can be interpreted as approximations of suitable transport maps. 

Clearly, given a dataset $\{(x_i,y_i)\}_{i=1}^{n}$ of independent and identically distributed (\textit{iid}) observations, it remains to specify how $\varphi$ is actually implemented and trained. We address these aspects in the next subsection.

\subsection{Model design and training}
\label{subsec:training}
We now come to the practical implementation of CPFNs. Specifically, we shall discuss how to design and train such architectures. 

\paragraph{Design} Instead of constructing a single architecture from $\mathbb{R}^{d+q}\cong\mathbb{R}^d\times\mathbb{R}^q\to\mathbb{R}^q$, we shall rely on a separation of variables approach. Let $r\ge1$. The idea is to construct two neural network architectures
$$\phi:\mathbb{R}^d\to\mathbb{R}^{r\times q}\quad\text{and}\quad\psi:\mathbb{R}^q\to\mathbb{R}^{r\times q}$$
accepting $x\in\mathbb{R}^d$ and $u\in\mathbb{R}^q$ as inputs, respectively. Then, the generative model can be assembled via the formula
\begin{equation}
\label{eq:separation-of-variables}\varphi(x,u)=\left[\sum_{i=1}^{r}\phi_{i,1}(x)\psi_{i,1}(u),\dots,\sum_{i=1}^{r}\phi_{i,q}(x)\psi_{i,q}(u)\right]^\top,\end{equation}
where $\phi_{i,j}$ and $\psi_{i,j}$ correspond to the $i,j$ output of $\phi$ and $\psi$, respectively. 

The rationale for adopting this design is twofold. First, the inputs $ x $ and $ u $ may differ significantly in nature---for example, when their respective dimensions or statistical properties vary widely (e.g., $ d \gg q $ or vice versa). Processing them separately allows the network to form cleaner, more specialized representations before combining them, rather than entangling heterogeneous features from the outset. Second, the formulation in \eqref{eq:separation-of-variables} enhances interpretability, as it explicitly expresses the output $ Y $ as a linear combination of $ x $-dependent functions with random coefficients. Notably, if $ \psi $ is constant, \eqref{eq:separation-of-variables} becomes independent of $ u $, corresponding to the degenerate case where $ Y $ reduces to a deterministic function of $ X $.

In this work, we shall construct $\phi$ and $\psi$ using conventional deep feed forward architecture, assuming that $d$ and $q$ are not too large ---namely, $d,q\le 100.$ Higher-dimensional problems may require the use of more sophisticated models, such as convolutional and sparse layers, which are currently beyond the scope of this work.

The \emph{rank} hyperparameter, $r$, models the interaction between $x$ and $u$, and it can be tuned based on the problem at hand. Depth and width of the neural networks architectures, instead, can be chosen following the guidelines available in the deep learning literature, see e.g. \cite{guhring2021approximation} and \cite{adcock2021gap}. For instance, in our experiments we consider networks with 3 hidden layers of constant width, $l=50$, and a corresponding rank of comparable size, $r=50.$

\paragraph{Training} 
The training phase is the crucial stage that transforms $\varphi$ into a conditional push-forward neural network. Let $\{(x_i,y_i)\}_{i=1}^{n}\subset\mathbb{R}^d\times\mathbb{R}^q$ be a suitable training set, consisting of $n$ \textit{iid} realizations of $(X,Y)$. Let $\varphi_\theta:\mathbb{R}^d\times\mathbb{R}^q\to\mathbb{R}^q$ be a CPFN with submodules $\phi_\theta$ and $\psi_\theta$, respectively. Here, $\theta$ denotes the collection of all weights and biases 
$$\theta=(W_1,b_1,\dots,W_\ell,b_\ell)$$
appearing in the definition of $\varphi_{\theta}$. Notice that, strictly speaking, $\phi_\theta$ and $\psi_\theta$ should depend on two different subvectors of $\theta$, respectively. Here, we do not emphasize this distinction in order to ease notation.

Let $U\in\mathbb{R}^q$ be an auxiliary random variable, absolutely continuous and independent of $X$ and $Y$, e.g., $U\sim\mathcal{U}[0,1]^q$ or $U\sim\mathcal{N}(0,\mathbf{I}_q)$. The latter will be used to capture the variability in $Y$ that cannot be explained by $X$ alone ---in similar works, in fact, the latter is also referred to as \emph{explanatory variable}.
The idea is to leverage the training data in order to tune $\theta$ such that $\varphi_\theta(X, U)$ and $Y\,|\,X$ have approximately the same distribution.

To this end, we shall define a suitable loss function, whose purpose is to enforce the desired behavior of the generative model. Rather than presenting the final expression directly, here, we shall derive it through a sequence of heuristic arguments, which serve to clarify the underlying rationale. A more rigorous mathematical justification of this heuristic derivation is provided in Section~\ref{sec:theory}.
\\\\
To fix the ideas, let $U\sim\mathcal{U}[0,1]^q$.
Assume that the pair $(X,Y)$ is absolutely continuous, with density $f_{XY}$, so that for every $x$ in the support of $X$, the random variable $Y\,|\,X=x$ follows a probability law $\mathbb{P}_{Y|X=x}$ with density $f_{Y|X}(\cdot\,|\,x)=f_{XY}(x,\cdot)/f_{X}(x).$ 

To start, we notice that the generative model $\varphi_\theta(x,U)$ is not guaranteed to admit a density, which makes the comparison between the two distributions trickier. This is because, in contrast to normalizing flows, our construction makes no invertibility assumptions in the definition of the generative model $\varphi_\theta(x,\cdot)$, which does not allow us to apply the change of variables formula.
To overcome this difficulty, we shall introduce a hyperparameter $\varepsilon>0$, arbitrarily small, and an auxiliary random variable $Z$, independent of $(X,Y,U)$ with density $\kpdf:\mathbb{R}^q\to[0,+\infty).$ The idea is then to consider the following perturbed generative model
$$\varphi_\theta^\varepsilon(x,U)=\varphi_\theta(x,U)+\varepsilon Z$$
which, in contrast to the original one, is guaranteed to admit a density $f_{\theta,\varepsilon,x}$. The latter, in fact, is given by the convolution
\begin{equation}
    \label{eq:cpfn-density}
    f_{\theta,\varepsilon,x}(y)=\int_{[0,1]^q}\kpdf_\varepsilon(y-\varphi_\theta(x,u))du,
\end{equation}
where $\kpdf_\varepsilon(y)=\varepsilon^{-d}\kpdf(y/\varepsilon).$ In order to enforce the desired behavior of the generative model, we shall analyze the average discrepancy between the true conditional probability, $\mathbb{P}_{Y\,|\,X}$, and the predicted conditional law $\mathbb{P}_{\theta,\varepsilon,X}\sim f_{\theta,\varepsilon,X}.$ We measure such discrepancy using the Kullback–Leibler divergence, which brings us to consider the following error metric
$$E(\theta,\varepsilon)=\mathbb{E}_{X}\left[D_{\textnormal{KL}}\left(\mathbb{P}_{Y|X}\,||\,\mathbb{P}_{\theta,\varepsilon,X}\right)\right].$$
Expanding the above yields
\begin{align*}
E(\theta,\varepsilon)&=\mathbb{E}_{X}\left[\mathbb{E}_{Y|X}\left(\log \frac{f_{Y|X}(Y|X)}{f_{\theta,\varepsilon,X}(Y|X)}\right)\right]\\&=\mathbb{E}_{X}\left[\mathbb{E}_{Y|X}\left(\log f_{Y|X}(Y|X)\right)\right]-\mathbb{E}_{X}\left[\mathbb{E}_{Y|X}\left(\log f_{\theta,\varepsilon,X}(Y|X)\right)\right].\end{align*}
Since the first term is independent of $\theta$, we shall focus on the second one. By applying the tower property and a Monte Carlo approximation we obtain
$$-\mathbb{E}_{X}\left[\mathbb{E}_{Y|X}\left(\log f_{\theta,\varepsilon,X}(Y|X)\right)\right]=-\mathbb{E}\left[\log f_{\theta,\varepsilon,X}(Y|X)\right]\approx -\frac{1}{n}\sum_{i=1}^{n}\log f_{\theta,\varepsilon,x_i}(y_i),$$
which is essentially the negative log-likelihood of the data, measured according to the predicted density.  
Therefore, in order to push $\mathbb{P}_{\theta,\varepsilon, X}$ towards $\mathbb{P}_{Y|X}$, the above reasoning suggests training the generative model by minimizing the following loss function
\begin{equation}
\label{eq:integral-loss}
\mathscr{L}_{n}^{\varepsilon}(\theta)=-\frac{1}{n}\sum_{i=1}^{n}\log\left(\int_{[0,1]^q}\kpdf_\varepsilon(y_i-\varphi_\theta(x_i,u))du\right).
\end{equation}

In practice, in order to compute the integral term and stabilize the behavior of the logarithm near zero we replace the above with
\begin{equation}
\label{eq:true-empirical-loss}
\mathscr{L}_{n}^{\varepsilon,\delta,R}(\theta)=-\frac{1}{n}\sum_{i=1}^{n}\log\left[\delta + \frac{1}{R}\sum_{j=1}^{R}\kpdf_\varepsilon(y_i-\varphi_\theta(x_i,u_{i,j}))\right].
\end{equation}
where $\delta>0$ is a suitable hyperparameter, arbitrarily small, whereas $\{u_{i,j}\}_{j=1}^{R}$ are random points sampled uniformly from the unit cube $[0,1]^q$. The latter need not to be fixed and can be re-drawn at every iteration of the training procedure, fostering generalizability.

\paragraph{Summary} To summarize, given a dataset $\{(x_i,y_i)\}_{i=1}^n$ of $n$
\textit{iid} realizations of $(X,Y)$, the implementation of a CPFN entails:
\begin{itemize}
\item[\emph{(i)}] initializing a suitable architecture $\varphi_\theta:\mathbb{R}^{p}\times\mathbb{R}^{q}\to\mathbb{R}^q$ with submodules $\phi_\theta$ and $\psi_\theta$;
\item[\emph{(ii)}] defining a probability distribution $\mathbb{Q}$ over $\mathbb{R}^q$ to sample from (uniform, gaussian, etc.);
\item[\emph{(iii)}] fixing an auxiliary density, $\kpdf$, operating as a mollifier;
\item[\emph{(iv)}] choosing a bandwidth $\varepsilon>0$, an offset constant $\delta>0$, and a sample size $R$;
\item[\emph{(v)}] training the model by minimizing \eqref{eq:true-empirical-loss}.
\end{itemize}
Once the training phase is concluded the generative model is fully operational and it can be used to produce approximated samples of $Y\,|\,X.$ The latter can be exploited to explore the conditional law of $Y$ given $X$ and compute related quantities of interest.

\begin{remark}
\label{remark:optimization}
In practice, the bandwidth parameter $\varepsilon>0$ needs not to be fixed. Instead, it can be fine tuned during the training phase, by simultaneously minimizing \eqref{eq:true-empirical-loss} with respect to both $\theta$ and $\varepsilon.$ Finally, up to adapting the definition of $\kappa_\varepsilon$ in the obvious way, one can also use a vector of bandwidths $\varepsilon=[\varepsilon_1,\dots,\varepsilon_q]^\top$.   
\end{remark}

\begin{remark}
\label{remark:quadrature}
The collocation points used to approximate the integral in \eqref{eq:integral-loss} 
may be selected either randomly or deterministically, for instance through suitable 
quadrature rules. In the latter case, \eqref{eq:true-empirical-loss} is slightly 
modified, with the sum over $j$ weighted by the corresponding quadrature coefficients. 
Importantly, in certain situations it is possible to design tailored quadrature rules 
such that \eqref{eq:integral-loss} is evaluated exactly rather than approximately. 
For example, when using popular nonlinearities such as the ReLU or LeakyReLU activation, 
the mapping $\varphi_\theta=\varphi_\theta(x,u)$ is piecewise quadratic in $u$. If, in addition, 
a piecewise polynomial kernel $\kpdf$ is employed---such as the Epanechnikov or Quartic 
kernel---then the function 
$
u \mapsto \kpdf_\varepsilon\!\left(y_i-\varphi_\theta(x_i,u)\right)
$
is piecewise polynomial for each $i=1,\dots,n$. Consequently, the integral can be computed 
exactly using only a finite number of collocation points $R$, positioned dependending on $i$ and $\theta$.
\end{remark}

\subsection{Technical details: data normalization and rescaling}
\label{subsec:technical}
We conclude this section with some technical considerations concerning data normalization and rescaling. As it is well-known, due to the nonconvexity and nonlinearity of the loss function, the training dynamics of neural network architectures can be highly sensitive to data normalization or re-parametrization. In our setting, we notice that any equivalent representation of the covariates is admissible, meaning that if $b:\mathbb{R}^d\to\mathbb{R}^d$ is bijective, then one can replace $X$ with $\tilde{X}=b(X)$ without the need for special corrections. Conversely, re-parametrizations of the response variable need some additional care.

Specifically, assume that $\tilde{Y}=g(Y)$ for some diffeomorphism $g:\mathbb{R}^q\to\mathbb{R}^q$, where $g$ has been introduced to map response values onto a more convenient scale. For instance, a typical example is when $q=1$ and $Y$ is supported on $[0,A]$, with $A\gg 1$; if the values of the response variable have a strong skewness towards the left, a common transformation is $g(y)=\log(1+y).$
Assume now that we constructed a CPFN over the pair $(X,\tilde{Y})$. By leveraging \eqref{eq:cpfn-density}, we now have access to an estimated conditional density $\hat{f}_{\tilde{Y}|X}\approx f_{\tilde{Y}|X}.$ Then, we can exploit the change of variables formula to automatically derive an estimate for $f_{Y|X}$, namely
$$\hat{f}_{Y|X}(y\,|\,x)=|\nabla g(y)|\hat{f}_{\tilde{Y}|X}(g(y)\,|\,x),$$
where $|\nabla g(y)|$ denotes the determinant of the Jacobian of $g$ at $y.$

More in general, even when nonlinear transformations are not required, we found that the typical $z$-score standardization, $z\mapsto (z-\mu)/\sigma$, with estimated mean and standard deviation, can prove beneficial when applied to both $X$ and $Y$. For this reason, although the evidence is purely empirical, we recommend performing this step before designing and training a CPFN: see, e.g., our implementation at \url{github.com/NicolaRFranco/CPFN}.

\section{Theoretical analysis}
\label{sec:theory}
The goal of this section is to establish a convergence result supporting our construction in Section~\ref{sec:method}. In particular, we shall prove that, up to a bias depending on the constants $\varepsilon$ and $\delta$, the distribution produced by the generative model can actually converge to the true conditional law as $n\to\infty$, under suitable assumptions. Since $\varepsilon$ and $\delta$ can be chosen arbitrarily small, we will say that conditional push-forward neural networks are \emph{nearly} asymptotically consistent. 
The formal statement is given in Section~\ref{subsec:thm}, while the corresponding 
proof is deferred to the Appendix~\ref{subsec:proof}. 

We mention that, for the purposes of this section, the analysis has been simplified slightly by reformulating the 
loss function in terms of continuous integrals, thereby avoiding the explicit 
use of collocation points. We emphasize, however, that as discussed in 
Remark~\ref{remark:quadrature}, this simplification is not particularly restrictive and does 
not entail a loss of generality.
\\

\noindent In what follows, we denote by $\mathbb{N}$ the set of all natural numbers, including 0. Given $k\in\mathbb{N}$ and a multi-index $\alpha\in\{0,\dots, k\}^d$, we let $D^\alpha$ denote the differential operator
$$\frac{\partial^{\alpha.}}{\partial x_{1}^{\alpha_1}\dots\partial x_{d}^{\alpha_d}},$$
where $\alpha.=\sum_{i=1}^d\alpha_i.$
Given a nonempty regular open\footnote{We recall that a regular open set is an open set coinciding with the interior of its closure.} set $A\subset\mathbb{R}^d$, we write $C^k(\bar{A};\;\mathbb{R}^q)$ for the space of all vector-valued maps $f=(f_1,\dots,f_q)$ such that each $f_i$ is $(k-1)$-times differentiable in $A$, with Lipschitz continuous derivatives, and for which the norm
$$\|f\|_{C^k(\overline{A};\;\mathbb{R}^q)}=\sum_{i=1}^{q}\left[\max_{0\le\alpha.\le k-1}\;\sup_{a}|D^\alpha f_i(a)| + \max_{\alpha.=k-1}\sup_{a\neq b}\frac{|D^{\alpha}f_i(a)-D^\alpha f_i(b)|}{|b-a|}\right]$$
is finite, the suprema being for $a,b\in A$. 

\subsection{A result of near asymptotic consistency}
\label{subsec:thm}
Let $X$ and $Y$ be absolutely continuous random variables with supports in $\Xspace\subset\mathbb{R}^d$ and $\Yspace\subset\mathbb{R}^q$, respectively. Let $\XYspace\subseteq\mathbb{R}^d \times \mathbb{R}^q$ denote the joint support of $(X,Y).$
We make the following regularity assumptions, whose motivation and discussion are deferred to Remark~\ref{remark:assumptions}.

\begin{assumption}\label{assumption:support}
    $\XYspace$ is a compact, convex subset with non empty interior. The joint density $f_{XY}$ is an element of $C^k(\XYspace)$, with $k>(d+q)/2$, and it satisfies
    $$f_{XY}(x,y)>0\quad \forall (x,y)\in\XYspace.$$ 
\end{assumption}

Our approach requires specifying two constants, $\varepsilon,\delta \in (0,1)$, along with a kernel density function $\kpdf : \mathbb{R}^q \to [0, +\infty)$. Concerning the latter, we make the following assumption.

\begin{assumption}\label{assumption:kernel}
    $\kpdf$ is compactly supported in $[-1,1]^q.$ Furthermore, it is infinitely differentiable, with uniformly bounded derivatives up to order $k$.
\end{assumption}

As outlined in Section~\ref{subsec:training}, we also consider a random vector $Z\in\mathbb{R}^q$ with density $\kpdf$, independent of $X$ and $Y$. We denote by $\kpdf_\varepsilon(y)=\varepsilon^{-q}\kpdf_\varepsilon(y/\varepsilon)$ the scaled kernel. Given an activation $\rho:\mathbb{R}\to\mathbb{R}$, we write $\mathscr{A}_\rho(d,q)$ for the space of all CPFN architectures $\varphi:\mathbb{R}^{d}\times\mathbb{R}^{q}\to\mathbb{R}^q$ whose submodules employ $\rho$ as activation function in their hidden layers. We make the following assumption on $\rho$, which ensures a suitable expressivity and smoothness of the neural network models considered.

\begin{assumption}
    \label{assumption:rho}
    $\rho$ is a nonpolynomial $k$-times continuously differentiable function.
\end{assumption}

\noindent Given $M\ge1$, we consider the following hypothesis class.
$$
\mathscr{H}_{M}=\left\{\varphi\in\mathscr{A}_\rho(d,q)\;\text{such that}\;\int_{[0,1]^q}\|\varphi(\cdot,u)\|_{C^{k}(\Xspace;\mathbb{R}^q)}du< M\right\}.$$

For every $n$, let $\{(x_i,y_i)\}_{i=1}^{n}$ be an \textit{iid} random sample from $(X,Y)$, and consider the following empirical loss function 
$$\mathscr{L}_{n}^{\delta,\varepsilon}(\varphi)=-\frac{1}{n}\sum_{i=1}^n\log\left[\delta + \int_{[0,1]^q}\kpdf_\varepsilon(y_i-\varphi(x_i, u))du\right].$$
Let $\hat{\varphi}_{\varepsilon,\delta,n}\in\mathscr{H}_{M}$ be any $(1/n)$-quasi-minimizer of the loss function, that is, an element of the hypothesis class such that
$$\mathscr{L}^{\delta,\varepsilon}_n(\hat{\varphi}_{\varepsilon,\delta, n})\le\inf_{\varphi\in\mathscr{H}_{M}}\;\mathscr{L}_n^{\delta,\varepsilon}(\varphi)+\frac{1}{n}.$$
Let $\hat{f}_{\varepsilon,\delta,n,x}$ be the probability density function of the random variable $\hat{\varphi}_{\varepsilon,\delta,n}(x, U)+\varepsilon Z$, with $U\sim\mathcal{U}([0,1]^q)$ a uniform random variable independent of $X,Y,Z$. The following holds true.

\begin{theorem}[Near asymptotic consistency]
    \label{thm:main}
    Consider the setting described above.
    For $M$ large enough, one has $$\limsup_{n\to+\infty}\;\;\mathbb{E}_X^{1/2}\|f_{Y|X}-\hat{f}_{\varepsilon,\delta,n,X}\|_{L^1(\Yspace)}^2< c_1\sqrt{\delta}+c_2\sqrt{\varepsilon/\delta},$$
    for some positive constants $c_1,c_2>0$ independent of $\varepsilon$ and $\delta.$
\end{theorem}

Theorem~\ref{thm:main} establishes a weak form of consistency: while the estimator $\hat{f}_{\varepsilon,\delta,n,X}$ does not converge exactly to the true conditional density $f_{Y|X}$, the approximation error can be made arbitrarily small by choosing $\varepsilon,\delta\to0$ appropriately. Hence, the result guarantees asymptotic closeness in an averaged $L^1$ sense rather than full consistency. We conclude with a useful Corollary, descending directly from Theorem~\ref{thm:main}, which re-frames the result in terms of Wasserstain distances. To this end, we recall that the 1-Wasserstein distance $\mathcal{W}_1$ between two probability distributions with densities $f$ and $g$ is defined as
\begin{equation}
    \label{eq:wasserstein}
    \mathcal{W}_1(f,g)=
\inf\left\{
\int_{\mathbb{R}^q\times\mathbb{R}^q} |y - y'|\pi(dx,dy), \;\;\pi\in\Pi(f,g) \right\},
\end{equation}
where $\Pi(f,g)$ denotes the collection of all probability distributions over $\mathbb{R}^q\times\mathbb{R}^q$ whose marginals are $f$ and $g$, respectively. Intuitively, the 1-Wasserstein distance can be interpreted as quantifying the minimal transportation cost required to transform the mass distribution described by the density $ f$ into that described by $g$, where the cost of moving a unit mass from $y$ to $y'$ is given by the distance $|y - y'|$.

\begin{corollary}
\label{corollary:wasserstein}
Under the assumptions of Theorem~\ref{thm:main}, one has
$$\limsup_{n\to+\infty}\;\;\mathbb{E}_X\left[\mathcal{W}_1(f_{Y|X},\;\hat{f}_{\varepsilon,\delta,n,X})\right]< c_1'\sqrt{\delta}+c_2'\sqrt{\varepsilon/\delta},$$
for some positive constants $c_1',c_2'>0$ independent of $\varepsilon$ and $\delta.$
\end{corollary}

When considering compact domains, in fact, the 1-Wasserstein distance is bounded by the $L^1$ distance. This fact, combined with  Jensen's inequality, immediately yields the above.

\begin{remark}[Assumptions discussion] \label{remark:assumptions} We provide here some comments to the assumptions. The convexity requirement in Assumption~\ref{assumption:support} guarantees a well-behaved geometry of the support, ensuring that all conditional densities also have a convex support. This is essential for our construction, as it avoids the insurgence of singularities when computing the marginal quantile functions. Similar considerations apply to the strict separation assumption, $f_{XY}>0$ on $\XYspace.$ The compactness assumption, instead, is equivalent to the boundness of $(X,Y)$. The latter allows us to apply the classical results on universal approximation by deep neural networks, and it is crucial for deriving Corollary~\ref{corollary:wasserstein}. Conversely, the threshold on the regularity exponent, $k>(d+q)/2$, arises from the theory of empirical processes: this condition, in fact, guarantees that the corresponding function class is Glivenko--Cantelli (and in fact Donsker), which ensures uniform convergence of the empirical loss to its population counterpart.

 Assumption~\ref{assumption:kernel}, particularly the uniform boundness condition, ensures stability when taking derivatives under integrals, which is crucial in the approximation and bias analysis of the perturbed generative model. Conversely, the smoothness assumption on $\kappa$ guarantees a suitable regularity of the loss function, which is again essential when studying the limit for $n\to+\infty.$ Finally, Assumption~\ref{assumption:rho} guarantees compatibility between the smoothness of the generative model and the regularity of $f_{XY}$, while the nonpolynomiality condition ensures a suitable expressivity of the hypothesis class: see, e,g., Hornik’s universal approximation theorem \citep{hornik1991approximation}.
\end{remark}

\section{Simulation studies}
\label{sec:simulation} 

We begin our experimental analysis with two simulation studies designed to investigate the capabilities of CPFNs in a controlled setting. This approach allows us to directly assess the quality of the conditional distribution approximation using conditional samples—something typically infeasible in real-world case studies.

For this preliminary analysis, we restrict our comparison to CPFNs, KCDE, and, occasionally, quantile regression (QR). We deliberately avoid a broader benchmark including methods such as normalizing flows or tree-based algorithms, as conducting an unbiased comparison would require implementing and tuning these methods ourselves (a challenging task given that the considered case studies are novel and not covered in the existing literature). Instead, we defer this comparative analysis to Section~\ref{sec:test}, where the use of real-world datasets allows for a fair assessment based on previously published results.

%
%
We consider two simulation studies, encompassing both univariate and multivariate settings: see Sections~\ref{subsec:univariate} and~\ref{subsec:multivariate}, respectively. To evaluate the performance of an estimator $\hat f_{Y|X}$ of the conditional density $f_{Y|X}$, we use the average Wasserstein distance (AWD), defined as  
\begin{equation}\label{eq:AWD_uni}
    \operatorname{AWD}(f_{Y|X}, \hat f_{Y|X}) 
        = \mathbb{E}_X \left[ \mathcal{W}_1\big( f_{Y|X}(\cdot|X), \hat f_{Y|X}(\cdot|X) \big) \right],
\end{equation}
with $\mathcal{W}_1$ the 1-Wasserstein distance, cf. \eqref{eq:wasserstein}.
In our experiments, we estimate the latter via Monte Carlo, essentially comparing true and predicted conditional samples obtained for different values of $X$. For a more detailed explanation on the computation of the AWD, we refer the interested reader to the Appendix~\ref{appendix:awd-details}.


In the addition to the AWD, in the univariate case, we also assess the ability of our approach in recovering individual conditional quantile functions. Precisely, given a quantile level $\tau$, we consider the following average quantile error 
$$
\operatorname{AQE}(\tau) = \mathbb{E}_X \big| q_{Y|X}(\tau \mid X) - \hat q_{Y|X}(\tau \mid X) \big|
$$
where, as before, the expectation is approximated via Monte Carlo sampling ---with samples independent of the training set.   
This allows us to evaluate how well the proposed method reconstructs not only the full conditional distribution but also specific quantile functions. In this setting, beside KCDE, we also benchmark CPFN against classical QR methods. 

In fact, we push this comparison a step further by noting that, if one allows multiple fits of a single QR method ---obtained for varying quantile levels $\tau_j$---, then one can essentially obtain a QR-based global estimate of the whole conditional distribution. This allows us to compare CPFN and QR not only using the AQE but also by means of the AWD metric. In fact it is well-known that in the univariate case one has the identity 
\begin{equation}\label{eq:AWD_univariate}
    \operatorname{AWD}(f_{Y|X}, \hat f_{Y|X}) 
    = \mathbb{E}_X\left[ \int_0^1 
        \big| q_{Y|X}(\tau \mid X) - \hat q_{Y|X}(\tau \mid X) \big| \, \mathrm{d}\tau
    \right].
\end{equation}
Still, for a correct interpretation of the results, we emphasize that ---in principle--- this setup clearly favors QR methods. In fact, it allows the latter to use separate models, each specialized on a given quantile level, whereas CPFN and KCDE rely on a single unified fit. 



\subsection{Univariate test case}
\label{subsec:univariate}
We start by considering an univariate setting where both $Y$ and $X$ are one-dimensional random variables. The covariate follows a uniform distribution, $X \sim \mathcal{U}(0, 1)$, and the response variable $Y$ is defined as a nonlinear and heteroscedastic function of $X$, with additional randomness introduced through external variables, $B\sim\mathrm{Bernoulli}(1/2)$ and $W\sim\mathcal{N}(0,1)$, independent of $X$. Precisely, we consider a model in which 
the conditional distribution of $Y\mid X = x$ is given as
\begin{align}\label{eq:data_process1}
Y = 
\begin{cases}
10x(x - 0.5)(1.5 - x) + 0.3W(1.3 - x), & \text{if } x < 0.5 \text{ or } B = 0, \\
10x(x - 0.5)(0.8 - x) + 0.3W(1.3 - x), & \text{otherwise}.
\end{cases}
\end{align}
This piecewise structure induces conditional distributions of $Y \mid X = x$ that vary both in location and scale, 
exhibiting at times unimodal and bimodal behaviors.
As a result, the data-generating mechanism produces a rich and diverse set of distributional shapes across different regions of the covariate space providing a challenging and informative benchmark for evaluating the ability of various methods to recover the full conditional distribution. Figure~\ref{fig:simulated-data} shows a plot of a \textit{iid} sample of $(X,Y)$ of size $n = 1000$.  

\begin{figure}[h!]
\centering
\includegraphics[width=\textwidth]{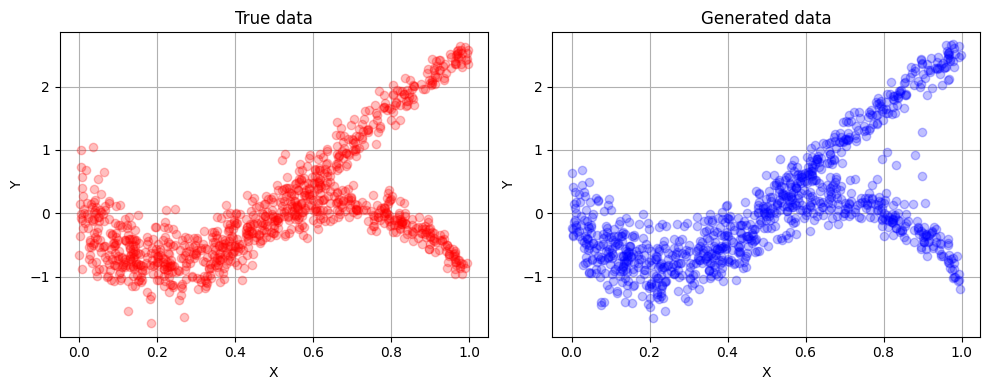}
\caption{\small \label{fig:simulated-data}True vs generated samples of $(X,Y$). Left: true samples, obtained via~\eqref{eq:data_process1}. Right: artificial samples, generated using the trained CPFN. Both samples contain $n = 1000$ observations.}
\end{figure}

In order to provide robust conclusions, we investigate the performances of CPFN and its competitors when fitted on datasets of variable sample sizes, $n=250,500,1000.$ Precisely, for each sample size, we generate the data, fit the models and evaluate the performance metrics multiple times, resulting in 100 independent Monte Carlo replicates for each $n$. This allows us to improve the reliability in our estimates by including Monte Carlo averages and standard deviations for the AWD and AQE.

We implement a CPFN of rank $r=20$ using two identical submodules $\phi$ and $\psi$, each consisting of 3 hidden layers of constant width $L=50$. We use the \emph{gelu} activation, which guarantees consistency with our theoretical assumptions, for all hidden layers, with the addition of the output layer for $\psi$. Finally, we use a gaussian distribution to sample the random features $U\sim\mathcal{N}(0,1)$ and we set $R=30$ during training; we also include $\varepsilon$ in the optimization process, starting with $\varepsilon_0=0.05$ as initial guess. We train for 3000 epochs using the Adam optimizer with default learning rate and set $\delta=10^{-15}$ in the definition of the loss function. Concerning QR methods, we explore a wide range of alternatives, each implemented through dedicated \texttt{R} packages. The classic \textit{linear} quantile regression is carried out using the \texttt{rq} function from the \texttt{quantreg} package~\citep{quantreg}. For a \textit{nearest-neighbour} approach, we employ the \texttt{qeKNN} function from the \texttt{qeML} package~\citep{qeML}. The \textit{spline} quantile smoothing is performed via the \texttt{qsreg} function in the \texttt{fields} package~\citep{fields}, which we preferred over the similar \texttt{rqss} function from the \texttt{quantreg} package due to its superior performance. Finally, we implement quantile regression using \textit{forests} via the \texttt{quantregForest} package~\citep{quantregForest}, offering a robust ensemble-based alternative. Concerning the KCDE method, we rely on the \texttt{R} package \texttt{np} \citep{npPackage}, which implements conditional distribution estimation, from which we derive the quantile function estimates through inversion.
 
Results are in Table~\ref{tab:sim_results} and Figures~\ref{fig:simulated-data}-\ref{fig:conditional-densities-1D}.


\begin{figure}[H]
    \centering
    \includegraphics[width=\linewidth]{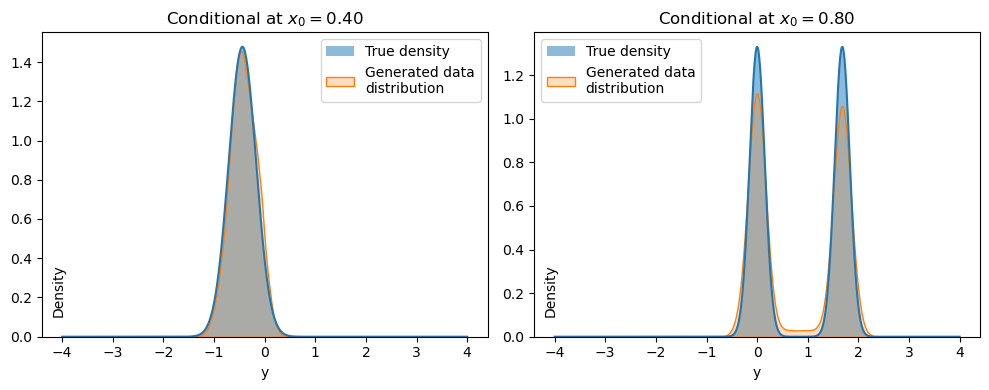}
    \caption{True vs predicted conditional density for two different covariate values $X=x_0.$}
    \label{fig:conditional-densities-1D}
\end{figure}

\;\\

\renewcommand{\arraystretch}{1.5}

\begin{table}[H]
\centering
\adjustbox{max width=\textwidth}{%
\begin{tabular}{|c|c|c|c|c|c|c|c|}
\hline
\textbf{Method} & $\boldsymbol{n}$ & $\textbf{AWD}$ & \textbf{AQE(0.10)} & \textbf{AQE(0.25)} & \textbf{AQE(0.50)} & \textbf{AQE(0.75)} & \textbf{AQE(0.90)}\\ \hline\hline
\multirow{3}{*}{$\begin{array}{cc}\text{CPFN}\\\text{(ours)}\end{array}$} 
  & 250  & 0.080 (0.019) & 0.063 (0.019) &  0.052 (0.015) &  0.285 (0.064) & 0.047 (0.015) & 0.060 (0.019) \\ \cline{2-8}
  & 500  & 0.061 (0.010) & 0.038 (0.013) &  0.038 (0.009) &  0.257 (0.057) &  0.038 (0.010) & 0.046 (0.014) \\ \cline{2-8}
  & 1000 & 0.051	(0.008) & 0.039 (0.009) & 0.030 (0.007) &  0.246 (0.052) &  0.030 (0.008) & 0.036 (0.010) \\ \hline\hline
\multirow{3}{*}{$\begin{array}{cc}\text{np.kde}\\\text{(KCDE)}\end{array}$} 
  & 250  & 0.125 (0.026) & 0.083 (0.020) & 0.076 (0.032) & 0.298 (0.044) & 0.084 (0.036) & 0.095 (0.023) \\ \cline{2-8}
  & 500  & 0.092 (0.019) & 0.063 (0.011) & 0.056 (0.022) & 0.289 (0.041) & 0.054 (0.013) & 0.067 (0.013) \\ \cline{2-8}
  & 1000 & 0.072 (0.013) & 0.048 (0.009) & 0.038 (0.007) & 0.274 (0.043) & 0.039 (0.009) & 0.049 (0.008) \\ \hline\hline
\multirow{3}{*}{linear QR} 
  & 250  & 0.385 (0.016) & 0.346 (0.026) & 0.307 (0.006) & 0.380 (0.012) & 0.380 (0.026) & 0.446 (0.038) \\ \cline{2-8}
  & 500  & 0.381 (0.010) & 0.343 (0.018) & 0.303 (0.004) & 0.375 (0.005) & 0.375 (0.016) & 0.448 (0.026) \\ \cline{2-8}
  & 1000 & 0.381 (0.007) & 0.342 (0.012) & 0.302 (0.002) & 0.373 (0.004) & 0.374 (0.011) & 0.450 (0.018) \\ \hline\hline
\multirow{3}{*}{$\begin{array}{cc}\text{qeML QR}\\\text{(nearest-neighbour)}\end{array}$} 
  & 250  & 0.131 (0.021) & 0.084 (0.013) & 0.087 (0.020) & 0.316 (0.040) & 0.088 (0.025) & 0.087 (0.013) \\ \cline{2-8}
  & 500  & 0.116 (0.016) & 0.072 (0.009) & 0.074 (0.025) & 0.327 (0.036) & 0.071 (0.015) & 0.072 (0.010) \\ \cline{2-8}
  & 1000 & 0.117 (0.012) & 0.069 (0.008) & 0.069 (0.016) & 0.326 (0.046) & 0.067 (0.015) & 0.069 (0.007) \\ \hline\hline
\multirow{3}{*}{$\begin{array}{cc}\text{fields QR}\\\text{(spline)}\end{array}$} 
  & 250  & 0.100 (0.023) & 0.062 (0.016) & 0.059 (0.016) & 0.301 (0.049) & 0.061 (0.017) & 0.064 (0.016) \\ \cline{2-8}
  & 500  & 0.068 (0.017) & 0.043 (0.011) & 0.041 (0.012) & 0.289 (0.048) & 0.041 (0.009) & 0.045 (0.012) \\ \cline{2-8}
  & 1000 & 0.051 (0.013) & 0.033 (0.008) & 0.028 (0.007) & 0.273 (0.049) & 0.028 (0.007) & 0.033 (0.009) \\ \hline\hline
\multirow{3}{*}{$\begin{array}{cc}\text{qrf QR}\\\text{(forests)}\end{array}$} 
  & 250  & 0.292 (0.023) & 0.190 (0.039) & 0.280 (0.062) & 0.446 (0.049) & 0.293 (0.065) & 0.195 (0.050) \\ \cline{2-8}
  & 500  & 0.292 (0.019) & 0.198 (0.039) & 0.285 (0.053) & 0.446 (0.050) & 0.286 (0.054) & 0.192 (0.035) \\ \cline{2-8}
  & 1000 & 0.293 (0.018) & 0.201 (0.033) & 0.283 (0.053) & 0.431 (0.055) & 0.286 (0.054) & 0.193 (0.028) \\ \hline
\end{tabular}
}
\caption{\small \small
Results for the univariate simulation study.
Model performance is measured by the AWD, 
and $\operatorname{AQE}(\tau)$ for $\tau \in \{0.1, 0.25, 0.5, 0.75, 0.9\}$. 
Cell values are averages over $m=100$ Monte Carlo replications, with standard deviations in parentheses.
\label{tab:sim_results}}
\end{table}

As seen in Table~\ref{tab:sim_results}, CPFN achieves the lowest AWD in all scenarios, indicating superior accuracy in recovering the full conditional distribution $f_{Y|X}$. The improvements are particularly notable for smaller samples, where the proposed estimator remains stable and consistently outperforms both kernel and nonparametric quantile regression alternatives.  
Among competing methods, the spline-based quantile regression (\texttt{fields}) and the kernel estimator (\texttt{np.kde}) exhibit the best performance, although both remain less accurate than our approach, especially at the distributional level as measured by the AWD. Linear quantile regression performs markedly worse, as expected, due to its inability to capture the nonlinear and heteroscedastic structure of the data-generating process~\eqref{eq:data_process1}. Nearest-neighbour (\texttt{qeML}) and forest-based (\texttt{qrf}) quantile regressions show moderate accuracy but display a higher variance across replicates.  

The $\operatorname{AQE}$ confirm these observations: our method yields smaller deviations from the true quantile functions for all quantile levels and sample sizes. 
As $n$ increases, all nonparametric estimators improve, yet the gap between our method and the alternatives remains consistent, highlighting its efficiency and robustness in univariate conditional density estimation.

It is interesting to note that all methods, CPFN included, struggle the most in estimating the median quantile. However, this is easily explained by the peculiar behavior of the true conditional distribution. We notice, in fact, that for $x>0.5$ the conditional distribution  $f_{Y\mid X=x}$ is bimodal and symmetric around the center of the two modes, see, e.g., Figure~\ref{fig:conditional-densities-1D}. In particular, as $x\to 1$, the gap between the two modes increases and the support of the true conditional density nearly becomes disconnected ---in principle this never happens, but numerically, it effectively does. Consequently, the position of the median becomes very hard to estimate, as multiple values around the true median yield very similar quantiles.

Relatedly, we notice that while CPFN manages to capture the change between unimodal and bimodal distributions, cf. Fig.~\ref{fig:conditional-densities-1D}, the approximation degrades slightly in the bimodal regime. We argue that this is because, as we discussed previously, there, the support of the conditional density is essentially disconnected, which poses significant challenges. Notice, in fact, that no continuous map can transform a gaussian ---or uniform--- distribution onto a distribution whose support has  multiple connected components, consistently with our theoretical assumptions in Section~\ref{sec:theory}. Since CPFNs are continuous, this scenario pushes their capabilities to the limit. However, we highlight that CPFNs ultimately manage to capture the bimodal behaviour regardless of this fact, by effectively learning to place a small mass in the intermediate region between the two modes, cf. Fig.~\ref{fig:conditional-densities-1D}.

\subsection{Multivariate test case}\label{subsec:multivariate}
We now consider a more challenging multivariate setting with a bivariate response $Y \in \mathbb{R}^2$ and a five-dimensional covariate $X \in \mathbb{R}^5$. In this example, the covariates determine the geometry of the conditional distribution $Y\mid X$, which is given by the superposition of a ``ring" component and two ``blob" components: see Figure~\ref{fig:multivariate}. The joint law of $(X, Y)$ is specified hierarchically, starting with
$$
X \sim \mathcal{N}(0, I_5),
$$
where $I_d$ is the identity matrix of dimension $d$. Given a realization $x \in \mathbb{R}^5$, we define the following latent parameters, which determine the weight and shape of the ring component,
$$
w_{\mathrm{ring}}(x) = s(\beta^\top x),\quad\quad
r_{\mathrm{ring}}(x) = r_0 +r_1\tanh(0.7\,\beta^\top x),\quad\quad
\sigma_{\mathrm{ring}}(x) = \sigma_0 + \sigma_1\,s(1.2\,\gamma^\top x),
$$
where $s$ is the sigmoid function, $r_0=2$, $r_1=1.5$, $\sigma_0=0.18$, $\sigma_1=0.15$, whereas
$$\beta = [1.2,\,-0.8,\,0.6,\,-0.4,\,0.9]^\top\quad\text{and}\quad\gamma = [-0.5,\,1.1,\,-0.3,\,0.7,\,0.4]^\top.$$
The blob components, instead, are parameterized by two means $m_1(x), m_2(x) \in \mathbb{R}^2$ and a shared covariance matrix $\Sigma(x)\in\mathbb{R}^{2\times2}$, each given by smooth transformations of $\gamma^\top x$.

A latent indicator variable then selects between the two components:
$$
Z \mid X=x \sim \mathrm{Bernoulli}(w_{\mathrm{ring}}(x)),
$$
so that $Z=1$ corresponds to the ring component and $Z=0$ to the blobs component. Specifically, when $Z=1$, we sample a radius
$
R \sim \mathcal{N}\bigl(r_{\mathrm{ring}}(x),\,\sigma_{\mathrm{ring}}^2(x)\bigr)$
truncated to $[0,\infty),
$
and an angle $\Theta \sim \mathcal{U}[0,2\pi)$, independently. 
Setting
$$
Y = (R\cos\Theta,\; R\sin\Theta)
$$
yields a ring-shaped conditional distribution with density
$$
f_{\mathrm{ring}}(y \mid x)
= \frac{g_x(\|y\|)}{2\pi\,\|y\|}, \qquad \|y\|>0,
$$
where $g_x$ is the truncated normal density of $R$.

When $Z=0$, instead, the response is drawn from a two-component Gaussian mixture with density
$$
f_{\mathrm{blobs}}(y \mid x)
= \tfrac{1}{2}\,\mathcal{N}\!\bigl(y; m_1(x), \Sigma(x)\bigr)
+ \tfrac{1}{2}\,\mathcal{N}\!\bigl(y; m_2(x), \Sigma(x)\bigr).
$$
Combining the two parts, the overall conditional distribution of $Y \mid X=x$ is
$$
f_{Y \mid X}(y \mid x)
= w_{\mathrm{ring}}(x)\, f_{\mathrm{ring}}(y \mid x)
+ \bigl(1 - w_{\mathrm{ring}}(x)\bigr)\, f_{\mathrm{blobs}}(y \mid x).
$$
As $x$ varies, the conditional shape of $Y \mid X=x$ transitions smoothly between ring-dominant, blob-dominant, and intermediate mixture regimes, yielding a rich and continuously deforming family of conditional densities.


\begin{figure}
    \centering
    \includegraphics[width=0.85\linewidth]{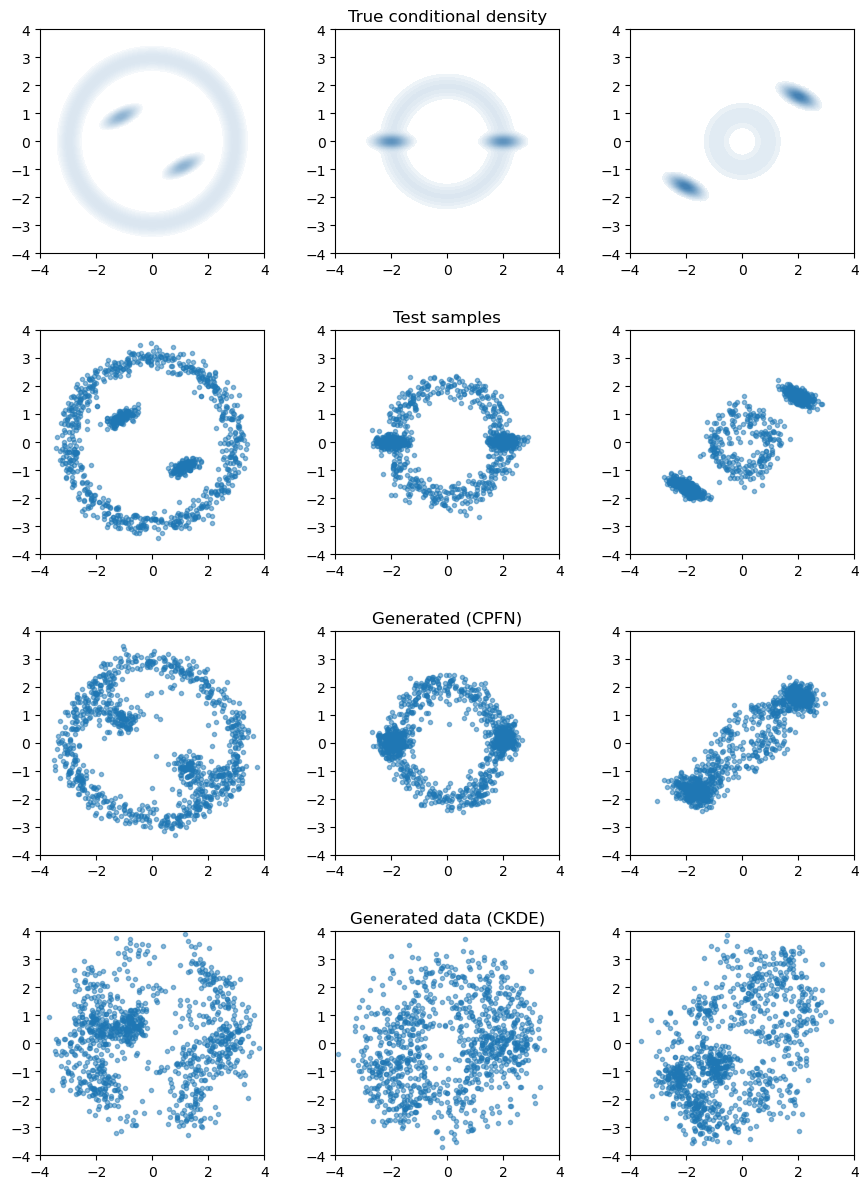}
\caption{\small 
        Comparison of conditional distributions in the multivariate test case for three representative covariate values:
        $x_{\mathrm{ring}}$ (ring-dominant), 
        $x_{\mathrm{trans}}$ (transition), and 
        $x_{\mathrm{blobs}}$ (blob-dominant). 
        Columns correspond to these three settings, while rows (from top to bottom) display: the true conditional density, $1000$ samples from the true conditional distribution, samples generated by the proposed CPFN model, and samples generated by the KCDE. 
        The proposed method accurately reproduces the multimodal ring– and blob–like geometries, while KCDE results appear more diffuse and less faithful to the true distributions.
    }
    \label{fig:multivariate}
\end{figure}

We generate the training data by drawing $n=5000$ \textit{iid} samples of $(X_i,Y_i)$ according to the procedure above, resulting in a dataset with seven columns—five covariate dimensions and two response coordinates. We leverage this data to fit a CPFN model following the procedure described in Section~\ref{sec:method}. Here, we use the same architecture employed for our previous test case, except for an increased rank, $r=50$, and a different size in the input layers. Recall in fact that, here, $d=5$ and $q=2$. In particular, we use $U\sim\mathcal{N}(0,I_2)$ and we set $R=100$ and $\varepsilon_0=[0.05, 0.05]^\top$ during training, cf. Remark~\ref{remark:optimization}. We train for 2000 epochs using Adam with default learning rate and we set $\delta=10^{-15}.$

Parallel to this, we also fit a KCDE model using the \texttt{ks} R package. We then compared the estimates produced by the two approaches using the 
$\operatorname{AWD}$ metric, 
cf. \eqref{eq:AWD_uni}. 
For additional details on this step, we refer the interested reader to
the Appendix~\ref{appendix:exp-details}.\vspace{-0.25cm}
\\\\
%
Results are in Figure~\ref{fig:multivariate} and Table~\ref{tab:multivariate_w1}. We immediately notice that, compared to the univariate test case, this problem is even more challenging. KCDE, in fact, fails to recover the structure of the conditional distributions, often returning very noisy results. This is particularly evident in Figure~\ref{fig:multivariate}, where we reported the conditional samples obtained for three different values of $X$, unseen during training, specifically,\vspace{-0.1cm}
$$
\begin{array}{llrrrrr}
    x_{\mathrm{ring}} & = & [\textcolor{white}{-}2.0, &\textcolor{white}{-}1.0,  &\textcolor{white}{-}0.5,  &-0.3,  &-1.0]^\top\\
    x_{\mathrm{trans}} & = & [\textcolor{white}{-}0.0,  &\textcolor{white}{-}0.0,  &\textcolor{white}{-}0.0, &\textcolor{white}{-}0.0,  &\textcolor{white}{-}0.0]^\top\\
    x_{\mathrm{blobs}} & = & [-2.0, &-1.0,  &-0.5,  &\textcolor{white}{-}0.3,  &\textcolor{white}{-}1.0]^\top
\end{array}
$$
which correspond, respectively, to ring-dominant, transitional, and blob-dominant regimes.

In contrast, CPFN manages to distinguish in between different scenarios, sometimes capturing the conditional distribution with remarkable accuracy, cf. Table~\ref{tab:multivariate_w1}. Still, the method faces some challenges for certain values of $X$, such as $x_{\mathrm{blobs}}$. As with our previous test case, we argue that this occurs because for such values of $X$ the support of the conditional distribution $Y \mid X$ becomes effectively disconnected (at least numerically). In particular, this scenario violates the assumptions in Theorem~\ref{thm:main}, which makes our findings consistent with the theory developed. 
Nonetheless, the samples produced by our method closely reproduce both the global structure and the multimodal geometry of the true conditional distributions, effectively capturing ring– and blob–like shapes as $x$ varies.

These qualitative differences are confirmed quantitatively by the estimated Wasserstein distances, cf. 
Table~\ref{tab:multivariate_w1}. As thereby reported, our method consistently achieves smaller discrepancies compared to the KCDE baseline, both for individual covariates and on average, confirming its superior accuracy in approximating the target conditional densities.

\begin{table}[H]
\centering
\begin{tabular}{l|l|l|l|l}
\hline
\textbf{Method} & \textbf{AWD} & $\mathcal{W}_1\mid X = x_{\mathrm{ring}}$ & $\mathcal{W}_1\mid X = x_{\mathrm{trans}}$ & $\mathcal{W}_1\mid X = x_{\mathrm{blobs}}$ \\ \hline
CPFN & 0.278 & 0.342 & 0.234 & 0.298\\ 
KCDE & 0.684 & 0.649 & 0.701 & 1.081\\ \hline 
\end{tabular}
\caption{\small Comparison between CPFN and KCDE for the multivariate test case. The two methods are compared using the AWD metric and the $\mathcal{W}_1$--distance between the predicted and true conditional distribution, here reported for three representative values of $X$.} 
\label{tab:multivariate_w1}
\end{table}


\section{
Experimental benchmarking on real-world data}\label{sec:test}
Building on our simulation studies, we evaluate the proposed CPFN on real-world datasets to examine its practical applicability and benchmark its performance against other state-of-the-art methods. To this end, we consider the experimental setup discussed in \cite[Section 6]{yang2024conditional}, as to avoid redesigning the experimental pipeline and instead enabling a like-for-like comparison.
%
%
In particular, following \cite{yang2024conditional}, we consider 14 different datasets extracted from the UCI repository \citep{frank2010uci}, of which Table~\ref{tab:datasets} provides a quick overview. 

\renewcommand{\arraystretch}{1.25}

\begin{table}[H]
\centering
\small
\begin{tabular}{ccc}
\begin{tabular}{lrrll}
\hline
\textbf{Dataset} & $\boldsymbol{n}$    & $\boldsymbol{d}$ & $X$ & $Y$\\
\hline
energy & 9568 & 4 & cont. & cont.\\
synchronous & 557 & 4 & cont. & cont.\\
localization & 107 & 5 & mixed & cont.\\
toxicity & 908 & 6 & mixed & cont. \\
concrete & 1030 & 8 & mixed & cont.\\
slump & 103 & 9 & cont. & cont.\\
forestfires & 517 & 12 & mixed & cont.\\
\hline
\end{tabular}
\begin{tabular}{lrr}
\,
\end{tabular}
\begin{tabular}{lrrll}
\hline
\textbf{Dataset} & $\boldsymbol{n}$    & $\boldsymbol{d}$ & $X$ & $Y$\\
\hline
navalpropolsion & 11934 & 16 & cont. & cont.\\
skillcraft & 3338 & 18 & mixed. & disc.\\
sml2010 & 2764 & 20 & cont. & cont.\\
thermography & 1003 & 25 & cont. & cont.\\
support2 & 9105 & 26 & mixed & cont.\\
studentmat & 395 & 43  & disc. & disc.\\
superconductivity & 21263 & 81 & cont. & cont.\\
\hline
\end{tabular}
\end{tabular}
\caption{\small
Datasets used in the experiments, sourced from the UCI repository. Columns are defined as follows: $n$ is the size of the sample, $d$ is the number of covariates (features), $X$ indicates the covariates' type, where ``cont.'' stands for continuous, ``disc.'' for discrete, and ``mixed'' denotes a combination of both ---the same notation is used to denote the type of the response variable, $Y$.}
\label{tab:datasets}
\end{table}

\noindent All case studies involve a univariate response, $q=1$, but have a variable number of covariates, $4\le d\le 81$. Since our approach was designed for continuous responses, we focus only on those datasets in which $Y$ is continuous (12 out of 14). At the same time, however, we place no restriction on the vector of covariates, allowing it to comprise both continuous and discrete features. This enables us to assess the flexibility of CPFN in handling more general scenarios, even beyond its original design.



To ensure a proper comparison, we fit CPFN on these datasets following the same procedure and using the same evaluation metric adopted in \cite{yang2024conditional}. In particular, we consider a five-fold cross-validation routine, where the fitted model is evaluated on the test set using the average negative log-likelihood (NLL), 
$$\operatorname{NLL}=-\frac{1}{N_{\text{test}}}\sum_{i=1}^{N_{\text{test}}}\log \hat{f}_{Y\mid X}(y_i^{\text{test}}\,|\,x_i^{\text{test}}),$$
which is arguably the primary evaluation metric when it comes to the literature on conditional density estimation \citep{yang2024conditional}. Relatedly, notice that here we cannot rely on the AWD metric to measure model performances as, differently from Section~\ref{sec:simulation}, we do not have access to conditional samples.

The benchmark models, taken from \cite{yang2024conditional}, include a variety of competitors such as NKDE and KCDE~\citep{rosenblatt1969conditional,rothfuss2019conditional}, which are based on kernel density estimation with bandwidth and neighborhood size tuned by cross-validation; LSCDE~\citep{sugiyama2010conditional}, which estimates the ratio between joint and marginal densities using Gaussian kernel basis functions; the tree-based CADET~\citep{cousins2019cadet}, which fits a Gaussian distribution on each leaf node; the recently proposed histogram-based CDE-HT~\citep{yang2024conditional}, which learns a tree structure where each leaf models a conditional density by a histogram and optimizes both the tree and histogram parameters jointly under the minimum description length principle; two additional tree-based baselines introduced by us (CART-k and CART-h) that respectively fit a kernel density model and a histogram after a CART regression tree~\citep{BreimanFriedmanOlshenStone1984} is learned from data; as well as neural network approaches NF~\citep{rezende2015variational} and MDN~\citep{pedregosa2011scikit} that apply dropout and noise regularization~\citep{rothfuss2019noise}, and the boosting-based LinCDE~\citep{gao2022lincde}. 

The source code for all benchmark models is available at \url{github.com/ylinicen/CDTree}, whereas our implementation of CPFN can be found at \url{github.com/NicolaRFranco/CPFN}. Here, in order to improve the interpretability of the results, we used the same setup for all of the case studies. Specifically, we always employed three-hidden layer gelu-networks for the two submodules of the CPFN, with terminal activation for $\psi$, constant width $L=50$ and rank $r=50.$ We used a gaussian distribution for sampling, $U\sim\mathcal{N}(0,1)$, and we set $R=100$ and $\varepsilon_0=0.05$ during training. To mitigate overfitting, for each fold in the cross-validation routine, we kept 10\% of the training data for validation. Specifically, we used such validation set to monitor the training dynamics and retain the most promising state of the CPFN across all epochs. All CPFNs were trained for 2000 epochs using the Adam optimizer with default learning rate. As for our previous case studies, we set $\delta=10^{-15}$. When addressing the datasets \emph{forestfires}, \emph{superconductivity} and \emph{support2}, we consider a log-transformation of the response variable, $y\mapsto\log(1+y)$. The NLL is then computed in the original scale using the change of variables formula, as discussed in Section~\ref{subsec:technical}.

Table~\ref{tab:results} presents the results reported in the original work~\cite{yang2024conditional}, along with those of CPFN for the considered test cases. 

\newcommand{\hide}[1]{}
\newcommand{\highlight}[1]{\underline{\textbf{#1}}}

\newcommand{\ourenergy}{2.73}
\newcommand{\oursynchrono}{-4.48}
\newcommand{\ourlocaliza}{4.19}
\newcommand{\ourtoxicity}{2.13}
\newcommand{\ourconcrete}{3.40}
\newcommand{\ourslump}{4.48}
\newcommand{\ourforestfires}{2.53}
\newcommand{\ournaval}{-5.61}
\newcommand{\oursml}{2.05}
\newcommand{\ourthermo}{0.91}
\newcommand{\oursupercond}{3.51}
\newcommand{\oursupport}{1.57}

\renewcommand{\arraystretch}{1.5}

\begin{table}[H]
\centering
\adjustbox{max width=\textwidth}{%
\begin{tabular}{l|ccccccccccc}
\textbf{Dataset} & \multicolumn{1}{c}{CADET} & \multicolumn{1}{c}{CART-h} & \multicolumn{1}{c}{CART-k} & \multicolumn{1}{c}{KCDE} & \multicolumn{1}{c}{LSCDE} & \multicolumn{1}{c}{NKDE} & \multicolumn{1}{c}{CDE-HT} & \multicolumn{1}{c}{LinCDE} & \multicolumn{1}{c}{MDN} & \multicolumn{1}{c}{NF} & \multicolumn{1}{c}{CPFN} \\ \hline
energy            & 3.55 & 3.09 & 3.06 & \highlight{2.47} & 3.38 & 3.00 & 2.93 & 2.93 & 2.78 & 2.86 & \ourenergy \\
synchron.         & -2.93 & -1.63 & -1.86 & -3.59 & -1.25 & -1.57 & -2.11 & -1.85 & -2.94 & -2.64 & \highlight{\oursynchrono} \\
localizat.        & -0.23 & -0.55 & -0.01 & -0.26 & -0.61 & -0.28 & -0.66 & \highlight{-0.95} & -0.68 & -0.43 & \ourlocaliza \\
toxicity          & 1.80 & 1.50 & 1.38 & 1.32 & 1.34 & 1.55 & 1.53 & 1.29 & 1.24 & \highlight{1.23} & \ourtoxicity \\
concrete          & 4.17 & 3.75 & 3.93 & 3.32 & 3.66 & 3.91 & 3.72 & 3.47 & \highlight{2.97} & 3.18 & \ourconcrete \\
slump             & 3.42 & 3.55 & 3.43 & 2.35 & 2.91 & 3.08 & 3.34 & 2.98 & \highlight{2.23} & 2.39 & \ourslump \\
forestfires         & \highlight{1.34} & 3.96 & 4.39 & 4.85 & 4.68 & 5.55 & 3.43 & 4.35 & 3.36 & 3.49 & \ourforestfires \\
navalprop.         & -3.53 & -3.33 & -3.66 & -3.56 & -2.88 & -3.19 & -3.30 & -3.16 & -4.12 & -3.75 & \highlight{\ournaval} \\
skillcraft         & 0.94 & 0.46 & -0.42 & 1.54 & 1.61 & 1.56 & \highlight{-1.02} & 1.26 & 0.35 & 1.11 & --- \\
sml2010           & 6.52 & 2.85 & 2.89 & \highlight{1.61} & 3.14 & 3.12 & 2.70 & 2.97 & 2.15 & 2.51 & \oursml \\
thermogr.         & 2.21 & 0.66 & 0.72 & 0.66 & 0.94 & 0.94 & 0.64 & 0.59 & 0.56 & \highlight{0.52} &  \ourthermo \\
support2          & 9.73 & 0.51 & 0.46 & 0.94 & 0.83 & 1.40 & \highlight{0.29} & 1.48 & 1.53 & 1.24 & \oursupport \\
studentmat         & 3.83 & \highlight{2.65} & 2.66 & 2.90 & 2.93 & 2.83 & 2.83 & 2.93 & 3.54 & 3.53 & --- \\
supercond.         & 9.60 & 3.84 & \highlight{3.46} & 4.45 & 4.17 & 4.19 & 3.48 & 3.87 & 3.94 & 3.85 & \oursupercond \\ \hline
\end{tabular}
}
\caption{\small NLLs for the considered test cases, adapted from \cite{yang2024conditional}. For each dataset, the results of the best performing model are highlighted in bold. Table entries refer to cross-validation averages over five folds.}
\label{tab:results}
\end{table}

\noindent 
Overall, we notice that CPFN proves highly competitive, at times surpassing the performances of all the considered benchmarks. 
It attains the best or near-best results on several datasets, including \textit{synchrono}, \textit{navalprop}, and \textit{supercond}, outperforming classical kernel-based estimators (e.g., KCDE, NKDE, LSCDE), as well as recent tree-based models (e.g., CDE-HT, CADET) but also well-established deep learning methods, such as normalizing flows (NF). On medium-sized datasets such as \textit{energy} and \textit{concrete}, CPFN achieves performance within a narrow margin of the top-performing methods, indicating a favorable balance between model flexibility and sample efficiency.

Still, there are also situations in which CPFN yields unsatisfactory results: see, e.g., the performances attained for \emph{toxicity}, \emph{slump}, and \emph{support2}. We believe, however, that this is to be expected. We notice, in fact, that none of the approaches considered in this analysis performs better than the others on more than two datasets. This goes to support the idea that selecting an appropriate approach requires first developing a solid understanding of the problem at hand, and that model selection is an intrinsically a data and context-dependent task.

We conclude with a short notice on the computational costs. Table~\ref{tab:times} reports the training time per fold in the cross-validation analysis across all datasets, comparing the resources required when training on a standard laptop (CPU) versus a high-performance GPU. We notice that, while training times inevitably increase with the number samples and the complexity of the architecture (whose design is tied to the number covariates), they always remain within an acceptable regime, confirming the feasibility of our approach in common applied settings.

\renewcommand{\arraystretch}{1.25}

\begin{table}[H]
\centering
\small
\begin{tabular}{lrrr}
\hline
\textbf{Dataset} & CPFN dofs & Training time (laptop CPU) & Training time (server GPU)\\
\hline
energy & 15650 & 45m 27s & 1m 24s \\
synchronous & 15650 & 3m 27s & 51s\\
localization & 15700 & 1m 32s & 23s\\
toxicity & 15750 & 4m 41s &  27s\\
concrete & 15850 & 5m 23s & 29s\\
slump & 15900 & 1m 37s & 24s\\
forestfires & 16050 & 3m 17s & 25s\\
navalpropolsion & 16150 & 51m 56s & 1m 24s\\
skillcraft & --- & --- & ---\\
sml2010 & 16300 & 12m 22s & 41s\\
thermography & 16700 & 5m 4s & 21s\\
support2 & 16750 & 40m 45s & 1m 2s\\
studentmat & --- & --- & ---\\
superconductivity & 19500 & 1h 37m 55s & 2m 7s\\
\hline
\end{tabular}
\caption{\small Number of trainable parameters in the CPFN architectures (dofs) and training times for the considered datasets. CPU times were obtained using a standard laptop with an 11th Gen Intel Core i7-1185G7 running at 3.00 GHz, whereas GPU times refer to a dedicated machine equipped with an NVIDIA Tesla V100 GPU and 64 MB of dedicated RAM.}
\label{tab:times}
\end{table}


\section{Conclusions}
\label{sec:conclusions}
We introduced conditional push-forward neural networks (CPFN), a simple yet expressive framework for conditional distribution estimation that shifts the focus from explicitly modelling $f_{Y|X}$ to learning a fast conditional sampler $u\mapsto\varphi(x,u)$. This generative viewpoint provides a practical route to explore the full law of $Y\mid X=x$ and to compute conditional summaries (e.g., means, variances, quantiles) via standard Monte Carlo without resorting to problem–specific quadrature rules.

Architecturally, CPFN is deliberately lightweight: a separation-of-variables design allows $x$- and $u$-representations to be constructed independently and combined through a low-rank interaction, while training reduces to a single-objective loss of negative log-likelihood type. 

On the theoretical side, we established a near-asymptotic consistency result showing that, up to controllable smoothing and offset biases, the conditional density induced by CPFN converges in an averaged $L^1$ sense to the true $f_{Y|X}$ as the sample size grows, with an immediate translation to 1-Wasserstein control. While intentionally modest, these guarantees anchor the method in classical statistical principles and clarify the role of the smoothing parameters.

Simulations show that CPFN accurately recovers complex, heteroscedastic and multimodal conditionals in both univariate and multivariate settings, 
surpassing strong nonparametric baselines, such as KCDE and spline/forest quantile regression methods.
On real data, CPFN emerges as a highly competitive method, delivering results that often match or outperform those achieved by leading state-of-the-art techniques, including tree-based and deep learning based approaches. 

Currently, the primary limitation of CPFNs stems from their convolution-based formulation, which can hinder scalability as the response dimension $q$ becomes large. In addition, their capability to represent conditional distributions with disconnected supports, particularly when relying on simple uniform or gaussian sources of stochasticity, remains limited.

These observations naturally suggest several extensions. First, coupling CPFN with structured architectures would widen applicability to images, sequences, or graphs, particularly when consider more complex covariates ($d\gg 1$). Conversely, the integration of suitable dimensionality reduction techniques may help mitigate the computational burden associated with evaluating the convolution integral in \eqref{eq:cpfn-density}, thereby improving scalability and stability when dealing with high-dimensional responses ($q \gg 1$).

Parallel to this, a second extension of CPFNs could be to consider more general architectures $\varphi:\mathbb{R}^d \times \mathbb{R}^u \to \mathbb{R}^q$ with $u \ge q$, where additional sources of stochasticity are injected in the generative model. In principle, in fact, this enrichment should mitigate the challenges posed by conditional densities with disconnected supports. At the same, however, this would raise the highly nontrivial question of determining $u$ given $X$ and $Y$.

Another promising research direction could be to consider alternative discrepancy measures beyond the Kullback--Leibler divergence, allowing for the treatment of censored / truncated responses, thus extending the applicability of CPFNs to problems in survival analysis.

Finally, we mention that this generative perspective also highlights applications such as probabilistic forecasting \citep{gneiting2014probabilistic}, conformal prediction \citep{fontana2023conformal}, uncertainty quantification \citep{soize2017uncertainty}, robust decision-making \citep{lempert2019robust}, and simulation-based inference \citep{cranmer2020frontier}, interesting research venues that we leave for future works.
\\

\paragraph{Acknowledgments}
NF acknowledges the support of project DREAM (Reduced Order Modeling and Deep Learning for the real-time approximation of PDEs), grant no. FIS00003154, funded by the Italian Science Fund (FIS) and by Ministero dell'Università e della Ricerca (MUR). NF and LT also acknowledge the support of the INdAM group "Gruppo Nazionale per il Calcolo Scientifico" (GNCS), of which they are members.

\bibliographystyle{plainnat}
\bibliography{references}

\begin{thebibliography}{64}
\providecommand{\natexlab}[1]{#1}
\providecommand{\url}[1]{\texttt{#1}}
\expandafter\ifx\csname urlstyle\endcsname\relax
  \providecommand{\doi}[1]{doi: #1}\else
  \providecommand{\doi}{doi: \begingroup \urlstyle{rm}\Url}\fi

\bibitem[Adcock and Dexter(2021)]{adcock2021gap}
Ben Adcock and Nick Dexter.
\newblock The gap between theory and practice in function approximation with deep neural networks.
\newblock \emph{SIAM Journal on Mathematics of Data Science}, 3\penalty0 (2):\penalty0 624--655, 2021.

\bibitem[Bashtannyk and Hyndman(2001)]{bashtannyk2001bandwidth}
David~M Bashtannyk and Rob~J Hyndman.
\newblock Bandwidth selection for kernel conditional density estimation.
\newblock \emph{Computational Statistics \& Data Analysis}, 36\penalty0 (3):\penalty0 279--298, 2001.

\bibitem[Baumann et~al.(2021)Baumann, Hothorn, and R{\"u}gamer]{baumann2021deep}
Philipp~FM Baumann, Torsten Hothorn, and David R{\"u}gamer.
\newblock Deep conditional transformation models.
\newblock In \emph{Joint European Conference on Machine Learning and Knowledge Discovery in Databases}, pages 3--18. Springer, 2021.

\bibitem[Bishop(1994)]{bishop1994}
Christopher~M Bishop.
\newblock Mixture density networks.
\newblock \emph{Technical report}, 1994.

\bibitem[Breiman et~al.(1984)Breiman, Friedman, Olshen, and Stone]{BreimanFriedmanOlshenStone1984}
L.~Breiman, J.~H. Friedman, R.~A. Olshen, and C.~J. Stone.
\newblock \emph{Classification and Regression Trees}.
\newblock Chapman \& Hall/CRC, 1st edition, 1984.
\newblock \doi{10.1201/9781315139470}.
\newblock URL \url{https://doi.org/10.1201/9781315139470}.

\bibitem[B{\"u}lte et~al.(2025)B{\"u}lte, Scholl, and Kutyniok]{bulte2025probabilistic}
Christopher B{\"u}lte, Philipp Scholl, and Gitta Kutyniok.
\newblock Probabilistic neural operators for functional uncertainty quantification.
\newblock \emph{arXiv preprint arXiv:2502.12902}, 2025.

\bibitem[Cousins and Riondato(2019)]{cousins2019cadet}
Cyrus Cousins and Matteo Riondato.
\newblock Cadet: interpretable parametric conditional density estimation with decision trees and forests.
\newblock \emph{Machine Learning}, 108\penalty0 (8):\penalty0 1613--1634, 2019.

\bibitem[Cranmer et~al.(2020)Cranmer, Brehmer, and Louppe]{cranmer2020frontier}
Kyle Cranmer, Johann Brehmer, and Gilles Louppe.
\newblock The frontier of simulation-based inference.
\newblock \emph{Proceedings of the National Academy of Sciences}, 117\penalty0 (48):\penalty0 30055--30062, 2020.

\bibitem[Devroye(2006)]{devroye2006nonuniform}
Luc Devroye.
\newblock Nonuniform random variate generation.
\newblock \emph{Handbooks in operations research and management science}, 13:\penalty0 83--121, 2006.

\bibitem[Dinh et~al.(2015)Dinh, Krueger, and Bengio]{dinh2015nice}
Laurent Dinh, David Krueger, and Yoshua Bengio.
\newblock Nice: Non‐linear independent components estimation.
\newblock In \emph{Proceedings of the International Conference on Learning Representations (ICLR) Workshop Track}, 2015.
\newblock arXiv:1410.8516.

\bibitem[Dinh et~al.(2017)Dinh, Sohl‐Dickstein, and Bengio]{dinh2017density}
Laurent Dinh, Jascha Sohl‐Dickstein, and Samy Bengio.
\newblock Density estimation using real nvp.
\newblock In \emph{Proceedings of the International Conference on Learning Representations (ICLR)}, 2017.
\newblock arXiv:1605.08803.

\bibitem[Evans(2018)]{evans2018measure}
LawrenceCraig Evans.
\newblock \emph{Measure theory and fine properties of functions}.
\newblock Routledge, 2018.

\bibitem[Fan et~al.(1996)Fan, Yao, and Tong]{fan1996}
Jianqing Fan, Qiwei Yao, and Howell Tong.
\newblock Estimation of conditional densities and sensitivity measures in nonlinear dynamical systems.
\newblock \emph{Biometrika}, 83\penalty0 (1):\penalty0 189--206, 1996.

\bibitem[Flamary et~al.(2021)Flamary, Courty, Gramfort, Alaya, Boisbunon, Chambon, Chapel, Corenflos, Fatras, Fournier, et~al.]{flamary2021pot}
R{\'e}mi Flamary, Nicolas Courty, Alexandre Gramfort, Mokhtar~Z Alaya, Aur{\'e}lie Boisbunon, Stanislas Chambon, Laetitia Chapel, Adrien Corenflos, Kilian Fatras, Nemo Fournier, et~al.
\newblock Pot: Python optimal transport.
\newblock \emph{Journal of Machine Learning Research}, 22\penalty0 (78):\penalty0 1--8, 2021.

\bibitem[Fontana et~al.(2023)Fontana, Zeni, and Vantini]{fontana2023conformal}
Matteo Fontana, Gianluca Zeni, and Simone Vantini.
\newblock Conformal prediction: a unified review of theory and new challenges.
\newblock \emph{Bernoulli}, 29\penalty0 (1):\penalty0 1--23, 2023.

\bibitem[Frank(2010)]{frank2010uci}
Andrew Frank.
\newblock Uci machine learning repository.
\newblock \emph{http://archive. ics. uci. edu/ml}, 2010.

\bibitem[Gao and Hastie(2022)]{gao2022lincde}
Zijun Gao and Trevor Hastie.
\newblock Lincde: conditional density estimation via lindsey's method.
\newblock \emph{Journal of machine learning research}, 23\penalty0 (52):\penalty0 1--55, 2022.

\bibitem[Gneiting and Katzfuss(2014)]{gneiting2014probabilistic}
Tilmann Gneiting and Matthias Katzfuss.
\newblock Probabilistic forecasting.
\newblock \emph{Annual Review of Statistics and Its Application}, 1\penalty0 (1):\penalty0 125--151, 2014.

\bibitem[Goodfellow et~al.(2016)Goodfellow, Bengio, Courville, and Bengio]{goodfellow2016deep}
Ian Goodfellow, Yoshua Bengio, Aaron Courville, and Yoshua Bengio.
\newblock \emph{Deep learning}, volume~1.
\newblock MIT press Cambridge, 2016.

\bibitem[Goodfellow et~al.(2014)Goodfellow, Pouget-Abadie, Mirza, Xu, Warde-Farley, Ozair, Courville, and Bengio]{goodfellow2014generative}
Ian~J Goodfellow, Jean Pouget-Abadie, Mehdi Mirza, Bing Xu, David Warde-Farley, Sherjil Ozair, Aaron Courville, and Yoshua Bengio.
\newblock Generative adversarial nets.
\newblock \emph{Advances in neural information processing systems}, 27, 2014.

\bibitem[Greenwell et~al.(2020)Greenwell, Racine, and Li]{npPackage}
Hayden Greenwell, Jeffrey~S. Racine, and Qi~Li.
\newblock \emph{np: Nonparametric Kernel Methods for Mixed Datatypes}.
\newblock CRAN, 2020.
\newblock URL \url{https://CRAN.R-project.org/package=np}.
\newblock R package version 0.60-12.

\bibitem[G{\"u}hring and Raslan(2021)]{guhring2021approximation}
Ingo G{\"u}hring and Mones Raslan.
\newblock Approximation rates for neural networks with encodable weights in smoothness spaces.
\newblock \emph{Neural Networks}, 134:\penalty0 107--130, 2021.

\bibitem[Hall et~al.(2004)Hall, Racine, and Li]{hall2004}
Peter Hall, Jeff Racine, and Qi~Li.
\newblock Cross-validation and the estimation of conditional probability densities.
\newblock \emph{Journal of the American Statistical Association}, 99\penalty0 (468):\penalty0 1015--1026, 2004.

\bibitem[Holmes et~al.(2010)Holmes, Gray, and Isbell~Jr]{holmes2010fast}
Michael~P Holmes, Alexander~G Gray, and Charles~Lee Isbell~Jr.
\newblock Fast kernel conditional density estimation: A dual-tree monte carlo approach.
\newblock \emph{Computational statistics \& data analysis}, 54\penalty0 (7):\penalty0 1707--1718, 2010.

\bibitem[Hornik(1991)]{hornik1991approximation}
Kurt Hornik.
\newblock Approximation capabilities of multilayer feedforward networks.
\newblock \emph{Neural networks}, 4\penalty0 (2):\penalty0 251--257, 1991.

\bibitem[Hothorn et~al.(2014)Hothorn, Kneib, and B{\"u}hlmann]{hothorn2014conditional}
Torsten Hothorn, Thomas Kneib, and Peter B{\"u}hlmann.
\newblock Conditional transformation models.
\newblock \emph{Journal of the Royal Statistical Society Series B: Statistical Methodology}, 76\penalty0 (1):\penalty0 3--27, 2014.

\bibitem[Huang and Boutros(2016)]{huang2016parameter}
Barbara~FF Huang and Paul~C Boutros.
\newblock The parameter sensitivity of random forests.
\newblock \emph{BMC bioinformatics}, 17\penalty0 (1):\penalty0 331, 2016.

\bibitem[Hyndman and Yao(2002)]{hyndman2002}
Rob~J Hyndman and Qiwei Yao.
\newblock Nonparametric estimation and symmetry tests for conditional density functions.
\newblock \emph{Journal of nonparametric statistics}, 14\penalty0 (3):\penalty0 259--278, 2002.

\bibitem[Hyndman et~al.(1996)Hyndman, Bashtannyk, and Grunwald]{hyndman1996}
Rob~J Hyndman, David~M Bashtannyk, and Gary~K Grunwald.
\newblock Estimating and visualizing conditional densities.
\newblock \emph{Journal of Computational and Graphical Statistics}, 5\penalty0 (4):\penalty0 315--336, 1996.

\bibitem[Hyt{\"o}nen et~al.(2016)Hyt{\"o}nen, Van~Neerven, Veraar, and Weis]{hytonen2016analysis}
Tuomas Hyt{\"o}nen, Jan Van~Neerven, Mark Veraar, and Lutz Weis.
\newblock \emph{Analysis in Banach spaces}, volume~12.
\newblock Springer, 2016.

\bibitem[Kanamori et~al.(2012)Kanamori, Suzuki, and Sugiyama]{kanamori2012}
Takafumi Kanamori, Taiji Suzuki, and Masashi Sugiyama.
\newblock Statistical analysis of kernel-based least-squares density-ratio estimation.
\newblock \emph{Machine Learning}, 86\penalty0 (3):\penalty0 335--367, 2012.

\bibitem[Kobyzev et~al.(2020)Kobyzev, Prince, and Brubaker]{kobyzev2020normalizing}
Ivan Kobyzev, Simon~JD Prince, and Marcus~A Brubaker.
\newblock Normalizing flows: An introduction and review of current methods.
\newblock \emph{IEEE transactions on pattern analysis and machine intelligence}, 43\penalty0 (11):\penalty0 3964--3979, 2020.

\bibitem[Koenker(2005)]{koenker2005quantile}
Roger Koenker.
\newblock \emph{Quantile regression}, volume~38.
\newblock Cambridge university press, 2005.

\bibitem[Koenker(2018)]{quantreg}
Roger Koenker.
\newblock \emph{quantreg: Quantile Regression}.
\newblock CRAN, 2018.
\newblock URL \url{https://CRAN.R-project.org/package=quantreg}.
\newblock R package version 5.36.

\bibitem[Koenker et~al.(1994)Koenker, Ng, and Portnoy]{koenker1994quantile}
Roger Koenker, Pin Ng, and Stephen Portnoy.
\newblock Quantile smoothing splines.
\newblock \emph{Biometrika}, 81\penalty0 (4):\penalty0 673--680, 1994.

\bibitem[Kone{\v{c}}n{\'a} and Horov{\'a}(2019)]{konevcna2019maximum}
Kate{\v{r}}ina Kone{\v{c}}n{\'a} and Ivanka Horov{\'a}.
\newblock Maximum likelihood method for bandwidth selection in kernel conditional density estimate.
\newblock \emph{Computational Statistics}, 34\penalty0 (4):\penalty0 1871--1887, 2019.

\bibitem[Kumar et~al.(2024)Kumar, Yang, and Lin]{kumar2024likelihood}
Shivam Kumar, Yun Yang, and Lizhen Lin.
\newblock A likelihood based approach to distribution regression using conditional deep generative models.
\newblock \emph{arXiv preprint arXiv:2410.02025}, 2024.

\bibitem[Lempert(2019)]{lempert2019robust}
Robert~J Lempert.
\newblock Robust decision making (rdm).
\newblock In \emph{Decision making under deep uncertainty: From theory to practice}, pages 23--51. Springer International Publishing Cham, 2019.

\bibitem[Matloff(2023)]{qeML}
Norm Matloff.
\newblock \emph{qeML: Quick and Easy Machine Learning Tools}, 2023.
\newblock URL \url{https://cran.r-project.org/package=qeML}.
\newblock R package version 1.1, published 2023-11-09.

\bibitem[Meinshausen(2017)]{quantregForest}
Nicolai Meinshausen.
\newblock \emph{quantregForest: Quantile Regression Forests}.
\newblock CRAN, 2017.
\newblock URL \url{https://CRAN.R-project.org/package=quantregForest}.
\newblock R package version 1.3-7.

\bibitem[Meinshausen and Ridgeway(2006)]{meinshausen2006quantile}
Nicolai Meinshausen and Greg Ridgeway.
\newblock Quantile regression forests.
\newblock \emph{Journal of machine learning research}, 7\penalty0 (6), 2006.

\bibitem[Nadaraya(1964)]{nadaraya1964estimating}
Elizbar~A Nadaraya.
\newblock On estimating regression.
\newblock \emph{Theory of Probability \& Its Applications}, 9\penalty0 (1):\penalty0 141--142, 1964.

\bibitem[Nychka et~al.(2025)Nychka, Furrer, Paige, Sain, Gerber, Iverson, and Johnson]{fields}
Douglas Nychka, Reinhard Furrer, John Paige, Stephan Sain, Florian Gerber, Matthew Iverson, and Rider Johnson.
\newblock \emph{fields: Tools for Spatial Data for Curve, Surface and Function Fitting}, 2025.
\newblock URL \url{https://cran.r-project.org/package=fields}.
\newblock R package version 16.3.1, published 2025-03-08.

\bibitem[Pacchiardi et~al.(2024)Pacchiardi, Adewoyin, Dueben, and Dutta]{pacchiardi2024probabilistic}
Lorenzo Pacchiardi, Rilwan~A Adewoyin, Peter Dueben, and Ritabrata Dutta.
\newblock Probabilistic forecasting with generative networks via scoring rule minimization.
\newblock \emph{Journal of Machine Learning Research}, 25\penalty0 (45):\penalty0 1--64, 2024.

\bibitem[Pedregosa et~al.(2011)Pedregosa, Varoquaux, Gramfort, Michel, Thirion, Grisel, Blondel, Prettenhofer, Weiss, Dubourg, et~al.]{pedregosa2011scikit}
Fabian Pedregosa, Ga{\"e}l Varoquaux, Alexandre Gramfort, Vincent Michel, Bertrand Thirion, Olivier Grisel, Mathieu Blondel, Peter Prettenhofer, Ron Weiss, Vincent Dubourg, et~al.
\newblock Scikit-learn: Machine learning in python.
\newblock \emph{the Journal of machine Learning research}, 12:\penalty0 2825--2830, 2011.

\bibitem[Pinkus(1999)]{pinkus1999approximation}
Allan Pinkus.
\newblock Approximation theory of the mlp model in neural networks.
\newblock \emph{Acta numerica}, 8:\penalty0 143--195, 1999.

\bibitem[Pospisil and Lee(2018)]{pospisil2018rfcde}
Taylor Pospisil and Ann~B Lee.
\newblock Rfcde: Random forests for conditional density estimation.
\newblock \emph{arXiv preprint arXiv:1804.05753}, 2018.

\bibitem[Rezende and Mohamed(2015)]{rezende2015variational}
Danilo Rezende and Shakir Mohamed.
\newblock Variational inference with normalizing flows.
\newblock In \emph{International conference on machine learning}, pages 1530--1538. PMLR, 2015.

\bibitem[Rosenblatt(1969)]{rosenblatt1969conditional}
Murray Rosenblatt.
\newblock Conditional probability density and regression estimators.
\newblock \emph{Multivariate analysis II}, 25:\penalty0 31, 1969.

\bibitem[Rothfuss et~al.(2019{\natexlab{a}})Rothfuss, Ferreira, Boehm, Walther, Ulrich, Asfour, and Krause]{rothfuss2019noise}
Jonas Rothfuss, Fabio Ferreira, Simon Boehm, Simon Walther, Maxim Ulrich, Tamim Asfour, and Andreas Krause.
\newblock Noise regularization for conditional density estimation.
\newblock \emph{arXiv preprint arXiv:1907.08982}, 2019{\natexlab{a}}.

\bibitem[Rothfuss et~al.(2019{\natexlab{b}})Rothfuss, Ferreira, Walther, and Ulrich]{rothfuss2019conditional}
Jonas Rothfuss, Fabio Ferreira, Simon Walther, and Maxim Ulrich.
\newblock Conditional density estimation with neural networks: Best practices and benchmarks.
\newblock \emph{arXiv preprint arXiv:1903.00954}, 2019{\natexlab{b}}.

\bibitem[Serfling(2002)]{serfling2002quantile}
Robert Serfling.
\newblock Quantile functions for multivariate analysis: approaches and applications.
\newblock \emph{Statistica Neerlandica}, 56\penalty0 (2):\penalty0 214--232, 2002.
\newblock \doi{https://doi.org/10.1111/1467-9574.00195}.
\newblock URL \url{https://onlinelibrary.wiley.com/doi/abs/10.1111/1467-9574.00195}.

\bibitem[Shen and Hsu(2020)]{shen2020conditional}
Pao-sheng Shen and Huichen Hsu.
\newblock Conditional maximum likelihood estimation for semiparametric transformation models with doubly truncated data.
\newblock \emph{Computational Statistics \& Data Analysis}, 144:\penalty0 106862, 2020.

\bibitem[Sohn et~al.(2015)Sohn, Lee, and Yan]{sohn2015learning}
Kihyuk Sohn, Honglak Lee, and Xinchen Yan.
\newblock Learning structured output representation using deep conditional generative models.
\newblock \emph{Advances in neural information processing systems}, 28, 2015.

\bibitem[Soize(2017)]{soize2017uncertainty}
Christian Soize.
\newblock \emph{Uncertainty quantification}, volume~23.
\newblock Springer, 2017.

\bibitem[Song et~al.(2025)Song, Wang, Shen, Lin, and Huang]{song2025wasserstein}
Shanshan Song, Tong Wang, Guohao Shen, Yuanyuan Lin, and Jian Huang.
\newblock Wasserstein generative regression.
\newblock \emph{Journal of the Royal Statistical Society Series B: Statistical Methodology}, page qkaf053, 2025.

\bibitem[Sugiyama et~al.(2010)Sugiyama, Takeuchi, Suzuki, Kanamori, Hachiya, and Okanohara]{sugiyama2010conditional}
Masashi Sugiyama, Ichiro Takeuchi, Taiji Suzuki, Takafumi Kanamori, Hirotaka Hachiya, and Daisuke Okanohara.
\newblock Conditional density estimation via least-squares density ratio estimation.
\newblock In \emph{Proceedings of the Thirteenth International Conference on Artificial Intelligence and Statistics}, pages 781--788. JMLR Workshop and Conference Proceedings, 2010.

\bibitem[Villani et~al.(2008)]{villani2008optimal}
C{\'e}dric Villani et~al.
\newblock \emph{Optimal transport: old and new}, volume 338.
\newblock Springer, 2008.

\bibitem[Wang et~al.(2024)Wang, Shin, and Bai]{wang2024generative}
Shijie Wang, Minsuk Shin, and Ray Bai.
\newblock Generative quantile regression with variability penalty.
\newblock \emph{Journal of Computational and Graphical Statistics}, 33\penalty0 (4):\penalty0 1202--1213, 2024.

\bibitem[Watson(1964)]{watson1964smooth}
Geoffrey~S Watson.
\newblock Smooth regression analysis.
\newblock \emph{Sankhy{\=a}: The Indian Journal of Statistics, Series A}, pages 359--372, 1964.

\bibitem[Wellner et~al.(2013)]{wellner2013weak}
Jon Wellner et~al.
\newblock \emph{Weak convergence and empirical processes: with applications to statistics}.
\newblock Springer Science \& Business Media, 2013.

\bibitem[Winkler et~al.(2019)Winkler, Worrall, Hoogeboom, and Welling]{winkler2019learning}
Christina Winkler, Daniel Worrall, Emiel Hoogeboom, and Max Welling.
\newblock Learning likelihoods with conditional normalizing flows.
\newblock \emph{arXiv preprint arXiv:1912.00042}, 2019.

\bibitem[Yang and van Leeuwen(2024)]{yang2024conditional}
Lincen Yang and Matthijs van Leeuwen.
\newblock Conditional density estimation with histogram trees.
\newblock \emph{Advances in Neural Information Processing Systems}, 37:\penalty0 117315--117339, 2024.

\bibitem[Zhao and Tabak(2025)]{zhao2025adaptive}
Wenjun Zhao and Esteban~G Tabak.
\newblock Adaptive kernel conditional density estimation.
\newblock \emph{Information and Inference: A Journal of the IMA}, 14\penalty0 (1):\penalty0 iaae037, 2025.

\end{thebibliography}
\appendix


\section{Proof of Theorem~\ref{thm:main}}
\label{subsec:proof}
We subdivide the proof into multiple steps. Before starting, however, we premit some notation concerning Bochner spaces. Precisely, given a separable Banach space $(\mathscr{V},\|\cdot\|_{\mathscr{V}})$, we write $L^1([0,1]^q;\;\mathscr{V})$ for the space of all Borel measurable maps $F:[0,1]^q\to\mathscr{V}$ with finite Bochner norm,
i.e. $\|F\|_{L^1([0,1]^q;\;\mathscr{V})}=\int_{[0,1]^q}\|F(u)\|_{\mathscr{V}}du<+\infty.$ For all such maps, we note by $\int_{[0,1]^q}F(u)du$ the unique element $\overline{F}\in\mathscr{V}$ satisfying
$$\ell(\overline{F})=\int_{[0,1]^q}\ell(F(u))du\quad\quad\forall\ell\in\mathscr{V}^*.$$
We recall that if $T:\mathscr{V}\to\mathscr{W}$ is a bounded linear operator between two separable Banach spaces, then, given any $F\in L^1([0,1]^q;\;\mathscr{V})$, one has $TF\in L^1([0,1]^q;\;\mathscr{W})$ and
$$\int_{[0,1]^q}TF(u)du=T\left(\int_{[0,1]^q}F(u)du\right),$$
see, e.g., \cite[Theorem 1.2.4]{hytonen2016analysis}.
\\\\
Coming back to the proof, let $d,q,X,Y,Z,U,k,\kpdf,\varepsilon,\delta,\rho,C,M,n, \XYspace, \Xspace$ be as in Section~\ref{subsec:thm}.

With little abuse of notation, given any $F\in L^1([0,1]^q;\mathcal{C}^{k}(\Xspace;\mathbb{R}^k))$, we shall write $F(x,u)$ to intend $F(u)(x).$

\paragraph{Step 1.} \textit{Define the following subset of $L^1([0,1]^q;\mathcal{C}^{k}(\Xspace;\mathbb{R}^k)),$
$$
\mathscr{G}_{M}=\left\{F\in L^1([0,1]^q;\mathcal{C}^{k}(\Xspace;\mathbb{R}^k))\;\text{such that}\;\int_{[0,1]^q}\|\varphi(\cdot,u)\|_{C^{k}(\Xspace;\mathbb{R}^q)}du< M\right\}.$$
For $M$ large enough, there exists $G\in\mathscr{G}_{M}$ such that $G(x,U)\sim Y\,|\,X=x$ for all $x\in\Xspace$}.


\begin{proof}
Thanks to the regularity assumptions on $\kpdf$, $f_X$ and $f_{XY}$, we notice that there exists $G\in L^{1}([0,1]^q;\;C^k(\Xspace;\;\mathbb{R}^q))$ such that, if $U\sim\mathcal{U}([0,1]^q)$ is a uniform random variable independent of $X,Y,Z$, then $G(x,U)\sim Y\,|\,X = x$
for all $x\in\Xspace.$ Such $G$ can be constructed using, e.g., Knothe-Rosenblatt transport \cite{villani2008optimal}, a generalization of the inverse transform method to random vectors, see Lemma~\ref{lemma:transport}. Therefore, setting $M>\|G\|_{L^{1}([0,1]^q;\;C^k(\Xspace;\;\mathbb{R}^q))}$ suffices.\end{proof}

\paragraph{Step 2} \textit{Up to a canonical embedding, $\mathscr{H}_{M}$ is a dense subset of $\mathscr{G}_{M}$.}

\begin{proof}
To start, we recall that $C^k([0,1]^q;\;C^k(\Xspace;\;\mathbb{R}^q))$ is dense in $L^1([0,1]^q;\;C^k(\Xspace;\;\mathbb{R}^q))$: see, for instance, the stronger result discussed in \cite[Proposition 2.4.23]{hytonen2016analysis}. Now, given any $F\in C^k([0,1]^q;\;C^k(\Xspace;\;\mathbb{R}^q))$, we notice that the map $(x,u)\mapsto F(x,u)$ is in $C^{k}(\Xspace\times [0,1]^q;\;\mathbb{R}^q).$ Consequently, the embedding 
$$C^{k}(\Xspace\times [0,1]^q;\;\mathbb{R}^q)\hookrightarrow L^1([0,1]^q;\;C^k(\Xspace;\;\mathbb{R}^q))$$
is guaranteed to produce a dense subset of $L^1([0,1]^q;\;C^k(\Xspace;\;\mathbb{R}^q)).$ At the same time, under Assumption~\ref{assumption:rho}, by classical results on universal approximation ---see, e.g., \cite[Theorem 4.1]{pinkus1999approximation}, $\mathscr{A}_\rho(d,q)$ is dense in $C^{k}(\Xspace;\;\mathbb{R}^q)\otimes C^k([0,1]^q;\;\mathbb{R}^q)$ under its canonical norm, which in turn, by compactness and convexity of both $\Xspace$ and $[0,1]^q$, is dense in $C^{k}(\Xspace\times [0,1]^q;\;\mathbb{R}^q)$. It follows that, up to embedding, $\mathscr{A}_\rho(d,q)$ is dense in $L^1([0,1]^q;\;C^k(\Xspace;\;\mathbb{R}^q)).$ Finally, notice that with little abuse of notation one has
$$\mathscr{H}_{M} = \mathscr{G}_{M}\cap \mathscr{A}_\rho(d,q).$$
Since $\mathscr{G}_{M}$ is open and nonempty, the conclusion follows.
\end{proof}

\paragraph{Step 3} \textit{Consider the family of functions
$$\mathscr{P}_\varepsilon=\left\{
(x,y)\mapsto \int_{[0,1]^q}\kpdf_\varepsilon(y-F(x,u))du\;\;:\;\;F\in\mathscr{G}_{M}
\right\}.$$
We have $\mathscr{P}_\varepsilon\subset C^k(\XYspace)$. Additionally, there exists some $c'=c'(q,k,\kpdf)>0$ such that $\|p\|_{C^k(\XYspace)}\le c'\varepsilon^{-q-k}M$ for all $ p\in\mathscr{P}_\varepsilon$.}

\begin{proof}
Notice that $\kpdf_\varepsilon\in C^k(\mathbb{R}^q)$ with $\|\kpdf_\varepsilon\|_{C^k(\mathbb{R}^q)}\le c\varepsilon^{-q-k}$ for some constant $c>0$ independent of $\varepsilon$. Then, given any $F\in\mathscr{G}_{M}$, by the chain rule and by linearity in $y$,
$$\int_{[0,1]^q}\|\kpdf_\varepsilon(\cdot-F(\cdot,u)\|_{C^k(\XYspace)}du\le c'\varepsilon^{-q-k}\int_{[0,1]^q}\|F(u)\|_{C^{k}(\Xspace;\;\mathbb{R}^q)}du<+\infty,$$
where $c'=c'(q,k,\kpdf).$
That is, the map $u\mapsto \kpdf_\varepsilon(y-F(x,u))$, the output being a function of $x$ and $y$, is an element of the Bochner space $L^1([0,1]^q;\;C^k(\XYspace))$. It follows that integrating over $u$ yields an element of $C^k(\XYspace).$ Additionally, the latter computation shows that
$$\|p\|_{C^k(\XYspace)}\le c'\varepsilon^{-q-k}M\quad\quad\forall p\in\mathscr{P}_\varepsilon.$$
\end{proof}

\paragraph{Step 4} \textit{The family
$$\mathscr{F}_{\delta,\varepsilon}=\{(x,y)\mapsto \log(\delta + p(x,y))\;\;:\;\;p\in\mathscr{P}_\varepsilon\},$$
satisfies the following bound on the bracketing entropy,
\begin{align}\label{eq:entropy_bound}
\log N_{[\;]}(\Delta,\mathscr{F}_{\delta,\varepsilon},L_2(\mathbb{Q}))\le C'\varepsilon^{-q-k}(\delta\Delta)^{-(d+q)/k},
\end{align}
for all probability distributions $\mathbb{Q}$ over $\XYspace$. Here, $C'=C'(d,q,k,\kpdf)>0$ is some positive constant independent of $\varepsilon$ and $\delta$.}

\begin{proof}
Since $\mathscr{P}_\varepsilon\subseteq\{p\in C^k(\XYspace)\;:\;\|p\|_{C^k(\XYspace)}\le c'\varepsilon^{-q-k}M\}$, 
by Theorem 2.7.1 in \cite{wellner2013weak}, we have the entropy bound
\begin{align}
\log N_{[\;]}(\Delta,\mathscr{P}_{\varepsilon},\|\cdot\|_{\infty})\le C'\varepsilon^{-q-k}\Delta^{-(d+q)/k},
\end{align}
for some $C'=C'(d,q,k,\kpdf,M)>0.$ We may assume $C'$ to be independent of $\varepsilon$ as the latter dependency may at most come from the size of the domain $|\XYspace|.$

We now notice that, given any probability distribution $\mathbb{Q}$, for every $p,p'\in\mathscr{P}_\varepsilon$ one has
$$\|\log(\delta+p)-\log(\delta+q)\|_{L_2(\mathbb{Q})}\le \frac{1}{\delta}\|p-q\|_{\infty}.$$
Indeed, for all nonnegative numbers $a,b\ge0$, by Lagrange's Theorem, one has
$$|\log(\delta+a)-\log(\delta+b)|\le\max_{c\ge\delta}\frac{1}{c}\cdot|b-a|=\frac{1}{\delta}|b-a|.$$
Thus, $(\mathscr{P}_\varepsilon,\|\cdot\|_{\infty})$ maps onto $(\mathscr{F}_{\delta,\varepsilon},L_2(\mathbb{Q}))$ with a Lipschitz constant of $\delta^{-1}$. The conclusion now follows by classical properties of entropy numbers, cf. \cite[Theorem 2.10.8]{wellner2013weak}.

\end{proof}

\paragraph{Step 5} \textit{Let $\mathbb{P}$ be the probability law of $(X,Y).$
The family
$\mathscr{F}_{\delta,\varepsilon}$ is a $\mathbb{P}$-Donsker class.}
\begin{proof} 
Under assumption $(iii)$, this is a direct consequence of the entropy bound in Eq. \eqref{eq:entropy_bound} combined with the Donsker Theorem \cite[Theorem 2.5.6]{wellner2013weak}.
\end{proof}

\paragraph{Step 6} \textit{For all $F\in\mathscr{G}_{M}$, define
$$\mathscr{L}^{\delta,\varepsilon}(F)=-\mathbb{E}\left[\log\left(\delta+\int_{[0,1]^q}\kpdf_\varepsilon(Y-F(X,u))du\right)\right].$$
Then, $\sup_{F\in \mathscr{G}_{M}}\left|\mathscr{L}^{\delta,\varepsilon}(F)-\mathscr{L}^{\delta,\varepsilon}_n(F)\right|\to0$ as $n\to+\infty.$}
\begin{proof}
    Let $P$ denote the expectation operator with respect to $(X,Y)$, i.e. $P\psi=\mathbb{E}[\psi(X,Y)]$ for all $\psi\in C(\XYspace).$ Similarly, define the empirical expectation operator $P_n$ as $P_n(\psi)=\frac{1}{n}\sum_{i=1}^{n}\psi(x_i,y_i).$ Since
    $$\sup_{F\in \mathscr{G}_{M}}\left|\mathscr{L}^{\delta,\varepsilon}(F)-\mathscr{L}^{\delta,\varepsilon}_n(F)\right|=\sup_{\psi\in\mathscr{F}_{\delta,\varepsilon}}|P\psi-P_n\psi|,$$
    the conclusion follows directly from Step 5.
\end{proof}

\paragraph{Step 7} \textit{$\mathbb{P}_\varepsilon$-almost surely, we have $\limsup_{n\to+\infty}\mathscr{L}^{\delta,\varepsilon}(\hat{\varphi}_{\varepsilon,\delta,n})\le \mathscr{L}^{\delta,\varepsilon}(G)$.
}
\begin{proof}
First of all, notice that due density of $\mathscr{H}_{M}$ in $\mathscr{G}_{M}$ and continuity of $\mathscr{L}_n^{\delta,\varepsilon}$, we have $$\mathscr{L}_n^{\delta,\varepsilon}(\hat{\varphi}_{\varepsilon,\delta,n})\le \inf_{F\in\mathscr{G}_{M}}\mathscr{L}_n^{\delta,\varepsilon}(F)+1/n\le \mathscr{L}_n^{\delta,\varepsilon}(G)+1/n.$$
Consequently,
    \begin{align*}
        \mathscr{L}^{\delta,\varepsilon}(\hat{\varphi}_{\varepsilon,\delta,n})&\;\le\; \mathscr{L}_n^{\delta,\varepsilon}(\hat{\varphi}_{\varepsilon,\delta,n})+|\mathscr{L}^{\delta,\varepsilon}(\hat{\varphi}_{\varepsilon,\delta,n})-\mathscr{L}_n^{\delta,\varepsilon}(\hat{\varphi}_{\varepsilon,\delta,n})|\;\\
        &\;\le\;
        \mathscr{L}_n^{\delta,\varepsilon}(G)+1/n+\sup_{F\in \mathscr{G}_{M}}\left|\mathscr{L}^{\delta,\varepsilon}(F)-\mathscr{L}^{\delta,\varepsilon}_n(F)\right|.
    \end{align*}
Then, taking the limit as $n\to+\infty$ gives the desired inequality.
\end{proof}

\paragraph{Step 8} \textit{Let $\tilde{Y}=Y+\varepsilon Z$. There exists some $\tilde{c}>0$, independent of $\varepsilon$ and $\delta$, such that for every $(x,y)\in\XYspace$ one has
$$\int_{\mathbb{R}^q}f_{Y|X}(y\;|\;x)|f_{Y|X}(y\;|\;x)-f_{\tilde{Y}|X}(y\;|\;x)|\le \tilde{c}\varepsilon.$$
}
\begin{proof}
    Notice that, by independence of $Y$ and $Z$, for every $x\in \Xspace$, the perturbed density $f_{\tilde{Y}|X}(\cdot\;|\;x)$ is the convolution of $f_{Y|X}(\cdot\;|\;x)$ with $\kappa$. Then, the above is just Lemma~\ref{lemma:convolution} applied to $g=f_{Y|X}(\cdot\;|\;x)$. Uniformity in $x$ comes from the fact that, since
    $$\min_{x\in \Xspace}f_{X}(x)=\min_{x\in \Xspace}\int_{\Yspace}f_{XY}(x,y)dy>\eta |\Yspace|>0$$
    for some $\eta>0$ (cf. Assumption 1), then $\|f_{Y|X}(\cdot\;|\;x)\|_{\infty}$ and $\text{Lip}(f_{Y|X}(\cdot\;|\;x))$ are uniformly bounded in $x$. Similarly, the good behavior at the boundary of $\XYspace$ transfers to that of the conditional supports, $\Yspace_x=\overline{\{y\in\Yspace\;:\;f_{XY}(x,y)>0\}}$, uniformly in $x.$
\end{proof}

\paragraph{Step 9} \textit{For every $x\in\Xspace$, let $\hat{f}_{\varepsilon,\delta,n,x}$ be the density of $\hat{\varphi}_{\varepsilon,\delta,n}(x,U)+\varepsilon Z.$
Then,
$$\limsup_{n\to+\infty}\;\int_{\XYspace}f_{XY}(x,y)\log \left(\frac{f_{Y|X}(y\,|\,x)}{\delta+\hat{f}_{\varepsilon,\delta,n,x}(y)}\right)dydx\le \tilde{c}\varepsilon\delta^{-1}.$$
} 

\begin{proof}
    To start, recall that by construction we have $\hat{f}_{\varepsilon,\delta,n,x}(y)=\int_{[0,1]^q}\kpdf_\varepsilon(y-\hat{\varphi}_{\varepsilon,\delta,n}(x,u))du$.
    Next, notice that by adding $\int_{\XYspace}f_{XY}(x,y)\log f_{Y|X}(y\,|\,x)dydx$ to the inequality in Step 7 and expanding $\mathscr{L}^{\delta,\varepsilon}$, one obtains
    \begin{align*}
    \limsup_{n\to+\infty}\;\int_{\XYspace}f_{XY}(x,y)&\log \left(\frac{f_{Y|X}(y\,|\,x)}{\delta+\int_{[0,1]^q}\kpdf_\varepsilon(y-\hat{\varphi}_{\varepsilon,\delta,n}(x,u))du}\right)dydx\\
    &\le 
    \int_{\XYspace}f_{XY}(x,y)\log \left(\frac{f_{Y|X}(y\,|\,x)}{\delta+\int_{[0,1]^q}\kpdf_\varepsilon(y-G(x,u))du}\right)dydx\\
    &\le 
    \int_{\XYspace}f_{XY}(x,y)\log \left(\frac{\delta+f_{Y|X}(y\,|\,x)}{\delta+\int_{[0,1]^q}\kpdf_\varepsilon(y-G(x,u))du}\right)dydx.
    \end{align*}
    We now notice that $G(x,U)+\varepsilon Z\sim Y\mid X=x$, meaning that $\int_{[0,1]^q}\kpdf_\varepsilon(y-G(x,u))du=f_{\tilde{Y}|X}(y\,|\,x)$ for all $(x,y)\in\XYspace$. Consequently,
    \begin{multline*}
        \int_{\XYspace}f_{XY}(x,y)\log \left(\frac{\delta+f_{Y|X}(y\,|\,x)}{\delta+\int_{[0,1]^q}\kpdf_\varepsilon(y-G(x,u))du}\right)dydx\\=\int_{\XYspace}f_{XY}(x,y)\log \left(\frac{\delta+f_{Y|X}(y\,|\,x)}{\delta+f_{\tilde{Y}|X}(y\,|\,x)}\right)dxdy.
    \end{multline*}
   Next, we recall that, for $a,b\ge0$, we have $\log\left(\frac{\delta+a}{\delta+b}\right)\le|\log\left(\frac{\delta+a}{\delta+b}\right)|=|\log(\delta+a)-\log(\delta+b)|\le\delta^{-1}|b-a|.$ Thus,
    $$\int_{\XYspace}f_{XY}(x,y)\log \left(\frac{\delta+f_{Y|X}(y\,|\,x)}{\delta+f_{\tilde{Y}|X}(y\,|\,x)}\right)\le \delta^{-1}\int_{\XYspace}f_{XY}(x,y)|f_{Y|X}(y\;|\;x)-f_{\tilde{Y}|X}(y\;|\;x)|dxdy.$$
    However, by Step 8, the above can be bounded as
    \begin{align*}
        \delta^{-1}\int_{\XYspace}f_{XY}(x,y)&|f_{Y|X}(y\;|\;x)-f_{\tilde{Y}|X}(y\;|\;x)|dxdy\\
        &=\delta^{-1}\int_{\Xspace}f_{X}(x)\int_{\mathbb{R}^q}f_{Y|X}(y\;|\;x)|f_{Y|X}(y\;|\;x)-f_{\tilde{Y}|X}(y\;|\;x)|dxdy\\
        &\le \tilde{c}\delta^{-1}\varepsilon,
    \end{align*}
    as wished.
\end{proof}

\paragraph{Step 10} \textit{Theorem~\ref{thm:main} holds true.}
\begin{proof}
     We have
\begin{align*}
\mathbb{E}_X\|f_{Y\mid X}-\hat{f}_{\varepsilon,\delta,n,X}\|_{L^1(\Yspace)}^2&=\int_{\Xspace}f_{X}(x)\|f_{Y|X}(\cdot\;|\;x)-\hat{f}_{\varepsilon\delta,n,x}\|_{L^{1}(\Yspace)}^2dx\\
&\le\int_{\Xspace}f_{X}(x)\|f_{Y|X}(\cdot\;|\;x)-\hat{f}_{\varepsilon,\delta,n,x}\|_{L^{1}(\mathbb{R}^q)}^2dx.
\end{align*}
Pointwise in $x$, we may now exploit a tilted variant of the Pinsker's inequality ---cf. Lemma~\ref{lemma:tilted}--- 
to obtain
\begin{align*}
    \mathbb{E}_X\|f_{Y\mid X}-&\hat{f}_{\varepsilon,\delta,n,X}\|_{L^1(\Yspace)}^2\\
    &\le \int_{\Xspace}f_{X}(x)\left[2\delta|\Yspace|+8\delta^2|\Yspace|^2+4\int_{\mathbb{R}^q}f_{Y|X}(y\;|\;x)\log\left[\frac{f_{Y|X}(y\;|\;x)}{\delta + \hat{f}_{\varepsilon,\delta,n,x}(y)}\right]dy\right]dx\\&=
    2\delta|\Yspace|+8\delta^2|\Yspace|^2+4\int_{\Xspace}\int_{\mathbb{R}^q}f_{XY}(x,y)\log\left[\frac{f_{Y|X}(y\;|\;x)}{\delta + \hat{f}_{\varepsilon,\delta,n,x}(y)}\right]dydx.
    \end{align*}
In particular, by Step 9, taking the limsup as $n\to+\infty$ gives
$$\limsup_{n\to+\infty}\;\mathbb{E}_X^{1/2}\|f_{Y\mid X}-\hat{f}_{\varepsilon,\delta,n,X}\|_{L^1(\Yspace)}^2\le  \sqrt{2}\delta^{1/2}|\Yspace|^{1/2}+2\sqrt{2}\delta|\Yspace|+4\tilde{c}(\varepsilon/\delta)^{1/2}.$$
Since $\delta<1\implies \delta < \delta^{1/2}$, letting $c_1=\sqrt{2}|\Yspace_1|^{1/2}(1+2|\Yspace_1|)$ and $c_2=4\tilde{c}$ gives the desired bound and thus concludes the proof.
\end{proof}

\section{Auxiliary results}

\begin{lemma}[Smooth conditional Knothe-Rosenblatt rearrangement]
\label{lemma:transport}
Let $X\in\mathbb{R}^d$ and $Y\in\mathbb{R}^q$ be two absolutely continuous random vector with joint density $f_{XY}.$ Let $\mathcal{S}$ be the support of $f_{XY}$. Define $\Xspace=\{x\in\mathbb{R}^d\;:\;\{x\}\times\mathbb{R}^q\cap\mathcal{S}\neq\emptyset\}$. Finally,  let $U\sim\mathcal{U}([0,1]^q)$  be a uniform random variable, independent of $X$ and $Y$.
If:
\begin{itemize}
    \item[(i)] $\mathcal{S}$ is a compact, convex subset with nonempty interior,
    \item[(ii)] $f_{XY}\in\mathcal{C}^k(\mathcal{S})$,
    \item[(iii)] $\exists\eta>0$ such that $f_{XY}>\eta$ on $\mathcal{S}$,
\end{itemize} 
then, there exists $G\in L^{1}([0,1]^q;\;\mathcal{C}^{k}(\Xspace;\mathbb{R}^q))$ such that, for all $x\in\Xspace$, the random variables $$Y\mid X=x\quad\text{and}\quad G(x,U)$$ have the same distribution.
\end{lemma}
\begin{proof} 
Let $Y=(Y_1,\dots, Y_q)$. The idea of the proof is to: (1) discuss the interplay between the supports of the different random variables into play; (2) verify certain regularity conditions of the marginal and conditional distributions; (3) discuss how the latter regularity is preserved when performing a partial inversion; (4) leveraging such inversions to explicitly construct $G$. We discuss the several steps below.

\paragraph{Well behavior of the supports}
Let $\mathcal{S}_j$ be the support of $(X,Y_1,\dots,Y_j)$ and let $\mathcal{S}_j^c$ be that of $(Y_j+1,\dots, Y_d)$. With little abuse of notation, we also let $\mathcal{S}_0=\Xspace$ and $\mathcal{S}_d^c=\emptyset$.

Since $\mathcal{S}$ is convex, we notice that each $\mathcal{S}_j$ --which is merely a projection of $\mathcal{S}$-- is also convex. The same holds for all $\mathcal{S}_j^c$. Similarly, we notice that each $Y_j$ must have a convex nonempty support. In particular, since the latter is a subset of the real line, it must be of the form $[a_j,b_j]$ for some $a_j,b_j\in\mathbb{R}.$

\paragraph{Regularity of the marginal and conditional distributions}
For every $j=1,\dots, q-1$,  define the marginal density
\begin{equation}
    \label{eq:marginal}
    f_{j}(x,y_1,\dots,y_j)=\int_{\mathcal{S}_{j}^c}f_{XY}(x,y_1,\dots,y_{j},s_{j+1},\dots s_d)dxds_{j+1}\dots ds_{q}.
\end{equation}
If $j=q$, we simply set $f_q\equiv f_{XY}$, whereas, if $j=0$ we intend $f_0=f_0(x)$ to be the marginal density of $X$. Since $\mathcal{S}_j^c$ is compact, and the derivatives of $f_{XY}$ are uniformly bounded, we are allowed to differentiate under the integral sign. In particular, it is straighforward to see that $f_j\in\mathcal{C}^k(\mathcal{S}_j)$ for every $j=0,\dots, q.$ Additionally,
$$f_{j-1}(x,y_1,\dots,y_{j}) > \eta|\mathcal{S}_{j}^c|>0\quad\text{on}\quad\mathcal{S}_{j}.$$
This allows us to define the conditional density
$f_j^c:[a_j,b_j]\times \mathcal{S}_{j-1}\to(0,+\infty),$ as
$$f_j^c(y_j\;|\;x,y_1,\dots,y_{j-1})=\frac{f_j(x,y_1,\dots,y_j)}{f_{j-1}(x,y_1,\dots,y_{j-1})}.$$
Since the denominator of the above is strictly separated from 0, it is straightforward to see that 
$$f_j^c\in\mathcal{C}^k([a_j,b_j]\times\mathcal{S}_{j-1}).$$
Now, let $F_j^c$ be the conditional cumulative distribution function associated to $f_j^c$, that is
$$F_j^c(y_j\;|\;x,y_1,\dots,y_{j-1})=\int_{a_j}^{y_j}f_j^c(s\;|\;x,y_1,\dots,y_{j-1})ds.$$
Just as before, due regularity and compactness, we have $F_j^c\in\mathcal{C}^k([a_j,b_j]\times\mathcal{S}_{j-1})$.

\paragraph{Partial inversion} We now notice that, for all $j=0,\dots, q$, and all $(x,y_1,\dots,y_{j-1})\in\mathcal{S}_{j-1}$, the map
$$F_j^c(\;\cdot\;|\;x,y_1,\dots,y_{j-1}):\;[a_j,b_j]\to[0,1]$$
is strictly monotone over some subinterval of $I(x,y_1,\dots,y_{j-1})\subseteq [a_j,b_j]$, the latter being the support of
$$Y_j\mid X=x,\;Y_1 = y_1,\dots, Y_{j-1}=y_{j-1}.$$
In fact, for every $s\in I(x,y_1,\dots,y_{j-1})$, one has
$$[F_j^c(\;\cdot\;|\;x,y_1,\dots,y_{j-1})]'(s)=f_j^c(s\;|\;y_1,\dots,y_{j-1})=
\frac{f_j(x,y_1,\dots,y_{j-1},s)}{f_{j-1}(x,y_1,\dots,y_{j-1})}>0.$$
In particular, $F_j^c(\;\cdot\;|\;x,y_1,\dots,y_{j-1})$ is a surjection between $I(x,y_1,\dots, y_{j-1})$ and $(0, 1)$. We are then allowed to define the map
$G_j:(0,1)\times \mathcal{S}_{j-1}\to [a_j,b_j]$
as
$$G_j(u,x,y_1,\dots, y_{j-1})=[F_j^c(\;\cdot\;|\;x,y_1,\dots,y_{j-1})]^{-1}(u).$$
By the implicit function theorem, $G_j$ is then differentiable with
\begin{align}
    \label{eq:implicit_derivative}
    \frac{\partial G_j}{\partial y_k}(u,x,y_1,\dots,y_{j-1})&=-\frac{\displaystyle\frac{\partial F_j^c}{\partial y_k}(G_j(u,x,y_1,\dots,y_{j-1})\;|\;x,y_1,\dots,y_{j-1})}{\displaystyle\frac{\partial F_j^c}{\partial y_j}(G_j(u,x,y_1,\dots,y_{j-1})\;|\;x,y_1,\dots,y_{j-1})}\nonumber\\
    &=-\frac{\displaystyle\frac{\partial F_j^c}{\partial y_k}(G_j(u,x,y_1,\dots,y_{j-1})\;|\;x,y_1,\dots,y_{j-1})}{f_j^c(G_j(u,x,y_1,\dots,y_{j-1})\;|\;x,y_1,\dots,y_{j-1})},
\end{align}
for every $k=1,\dots, q-1$. Partial derivatives in $x$ follow similarly, whereas
\begin{equation}
    \label{eq:implicit_derivative2}
    \frac{\partial G_j}{\partial y_j}(u,x,y_1,\dots,y_{j-1})=-\frac{1}{f_j^c(G_j(u,x,y_1,\dots,y_{j-1})\;|\;x,y_1,\dots,y_{j-1})}.
\end{equation}
Once again, since the denominators of \eqref{eq:implicit_derivative} and \eqref{eq:implicit_derivative2} are strictly separated from 0, this suffices to prove that
$$G_j\in\mathcal{C}^{k}((0,1)\times\mathcal{S}_{j-1}).$$

\paragraph{Construction of $G$}
We define $G:\Xspace\times(0,1)^q\to\mathbb{R}^q$ as
\[
G(x, u_1, \dots, u_q)
= 
\begin{bmatrix}
    g_1(u_1, x) \\[4pt]
    g_2(u_2, u_1, x) \\[4pt]
    \vdots \\[4pt]
    g_q(u_q, \dots, u_1, x)
\end{bmatrix}
\]
where $g_1=G_1$, while, for $j=2,\dots,q$, one has
$$g_j(u_j,\dots,u_1,x)=G_j(u_j, x, g_1(u_1,x),\dots, g_{j-1}(u_{j-1},\dots,u_1,x)).$$
By composition, $G$ is $k$-times continuously differentiable. Furthermore, it is straightforward to see that there exists some $M$ such that, for all $(u_1,\dots,u_q)$, one has
$$\|G(\cdot,u_1,\dots,u_q)\|_{\mathcal{C}^k(\Xspace)}\le M.$$
In fact, the derivatives in $x$ of each $G_j$ can be bounded uniformly thanks to the compactness of $\mathcal{S}$, the smoothness of $f_{XY}$, and the strict separation $f_{XY}>\eta>0.$ In particular, with little abuse of notation, we have
$$G\in L^{\infty}((0,1)^q;\;\mathcal{C}^{k}(\Xspace)).$$
Since $\partial[0,1]^q$ has zero Lebesgue measure, it follows that $G\in L^{1}([0,1]^q\;\mathcal{C}^{k}(\Xspace)).$ 

\paragraph{Conditional distribution property}
To conclude, we shall prove that for every $x\in\Xspace$, one has $G(x,U)\sim Y\mid X=x.$ To see this, let $U=(U_1,\dots, U_q)$. For $j=1$, we have

$$\mathbb{P}([G(x,U)]_1 \le t) = \mathbb{P}(G_1(U_1,x)\le t)=\mathbb{P}(U_1\le F_1^c(t\;|\;x)) = F_1^c(t\;|\;x)$$
for all $t$ in the support of $Y_1\;|\;X=x.$ Consequently, $[G(x,U)]_1$ has the same distribution of $Y_1\mid X=x$, whos cumulative distribution function is precisely $F_1^c(\cdot\;|\;x).$

Proceeding by induction, assume we already know that the first $j-1$ coordinates of $G(x,U)$ are distributed as $(Y_1,\dots, Y_{j-1})\;|\;X=x.$ Conditionally on $[G(x,U)]_k=y_k$ for $k=1,\dots,j-1$, we have
\begin{align*}
\mathbb{P}([G(x,U)]_j \le t\;\;|\;\;[G(x,U)]_{k}=y_k\;\text{for}\;k=1,\dots,j-1) &= \mathbb{P}(G_j(U_j,x,y_1,\dots,y_{j-1})\le t)\\
&=\mathbb{P}(U_j\le F_j^c(t\;|\;x,y_1,\dots,y_k)\\&=F_j^c(t\;|\;x,y_1,\dots,y_{j-1}).
\end{align*}
In particular, for all $(y_1,\dots,y_{j-1})$, $[G(x,U)]_{j}$ and $Y_j$ have the same conditional distribution. Given our inductive hypothesis, this shows that the random vectors $([G(x,U)_{1},\dots[G(x,U)]_{j})$ and $(Y_1,\dots,Y_j)\;|\;X=x$ are identically distributed. Iterating the above argument until $j=q$ concludes the proof. 
\end{proof}

\begin{lemma}[Convolution error control]
\label{lemma:convolution}
Let $g$ be a density over $\mathbb{R}^q$. Let $\kappa:[-1,1]^q\to[0,+\infty)$ be another density. For every $\varepsilon$, let $g_\varepsilon=\kappa * g$ be given by the convolution
$$g_\varepsilon(y)=\int_{\mathbb{R}^q}\kappa(a)g(y-\varepsilon a)da.$$
Let $A$ be the support of $g$. Assume that $A$ is a compact subset with Lipschitz boundary. Further assume that $A$ coincides with the closure of its interior and that $g$ is Lipschitz continuous over $A$. Finally, assume that $\kappa$ has finite first moment. Then, there exists a constant $c>0$ such that
$$\int_{\mathbb{R}^q}g(y)|g(y)-g_{\varepsilon}(y)|dy\le c\varepsilon.$$
Such constant can be chosen only to be dependent on $\|g\|_{\infty}$, $\text{Lip}(g),\partial A$ and $q.$
\end{lemma}
\begin{proof}
    Let $B_\varepsilon=\{y\in A\;:\;\inf_{a'\in \partial A}|y-a'|<q^{1/2}\varepsilon\}$. Classical results in geometric measure theory ---see, e.g., \cite[Lemma 5.3]{evans2018measure}--- ensure that $|B_\varepsilon|\le c'\varepsilon$ for some $c'>0$ depending solely on $\partial A$ and $q.$
    
    Next, notice that, for every $y\in\mathbb{R}^q$, we have
    $$|g(y)-g_{\varepsilon}(y)|\le\int_{\mathbb{R}^q}k(a)|g(y)-g(y-\varepsilon a)|da.$$
    In particular, if $y\in A\setminus B_\varepsilon$, because $(y-\varepsilon a)\subset A$ for every $a\in[-1,1]^q$, then
    $$|g(y)-g_{\varepsilon}(y)|\le\text{Lip}(g)\varepsilon\int_{\mathbb{R}^q}k(a)|a|da = c''\varepsilon.$$
    Conversely, for all $y\in B_\varepsilon$, since $(y-\varepsilon a)$ may fall outside of $A$, we use the looser bound
    $$|g(y)-g_{\varepsilon}(y)|\le 2\|g\|_{\infty}\int_{\mathbb{R}^q}k(a)=2\|g\|_{\infty},$$
    where we recall that $g\in L^{\infty}(\mathbb{R}^q)$ due continuity over its support.
    Putting things together then yields
    \begin{multline*}
        \int_{\mathbb{R}^q}g(y)|g(y)-g_{\varepsilon}(y)| =
        \int_{B_\varepsilon}g(y)|g(y)-g_{\varepsilon}(y)|dy +
        \int_{A\setminus B_\varepsilon}g(y)|g(y)-g_{\varepsilon}(y)|dy\\
        \le 2\|g\|_{\infty}\int_{B_\varepsilon}g(y)dy + c''\varepsilon\int_{A\setminus B_\varepsilon}g(y)dy\le 2\|g\|_{\infty}^2|B_\varepsilon|+c''\varepsilon\le c\varepsilon.
    \end{multline*}
    
\end{proof}

\begin{lemma}[Tilted Pinsker-like inequality]
\label{lemma:tilted}
Let $f,g$ be two densities over $\mathbb{R}^q$. Let $A$ be the support of $f$, and let $|A|$ be its Lebesgue measure. Then, for every $\delta>0$ one has

$$\frac{1}{2}\|f-g\|_{L^{1}(\mathbb{R}^q)}^2\le \delta|A|+4\delta^2|A|^2+2\int_{\mathbb{R}^q}f(y)\log\left[\frac{f(y)}{\delta+g(y)}\right]dy.$$
\end{lemma}
\begin{proof}
    If $|A|=+\infty$, the inequality is trivial. Thus, let $|A|<+\infty.$ Define $$g_{\delta}(y)=\frac{1}{\delta|A|+1}\left[\delta\mathbf{1}_{A}(y)+g(y)\right].$$
    It is straightforward to see that $g_\delta$ is a density over $\mathbb{R}^q.$ The classical Pinsker's inequality gives
    \begin{align*}
    \frac{1}{2}\|f-g_\delta\|^2_{L^{1}(\mathbb{R}^q)}&\le
    \int_{\mathbb{R}^q}f(y)\log\left[\frac{f(y)}{g_\delta(y)}\right]dy=\\
    &=
    \int_{\mathbb{R}^q}f(y)\log\left[\frac{f(y)}{\delta+g(y)}\right]dy+\int_{\mathbb{R}^q}f(y)\log(\delta|A|+1)dy\le
    \\&\le
    \int_{\mathbb{R}^q}f(y)\log\left[\frac{f(y)}{\delta+g(y)}\right]dy+\delta|A|,
    \end{align*}
    where we used the fact that $f$ is supported over $A$, so that $g_\delta(y)=g(y)+\delta$ there, and $\log(1+z)\le z$ for $z\ge0.$

    Parallel to this, we shall now derive a lower-bound for $\|f-g_\delta\|_{L^{1}(\mathbb{R}^q)}.$ By reverse triangular inequality,
    \begin{align*}
        \|f-g_\delta\|_{L^{1}(\mathbb{R}^q)}\ge
        \|f-g\|_{L^{1}(\mathbb{R}^q)}
        -
        \|g-g_\delta\|_{L^{1}(\mathbb{R}^q)}.
    \end{align*}
    We now notice that, for all $y\in\mathbb{R}^q$,
    $$|g(y)-g_{\delta}(y)|=\frac{1}{\delta|A|+1}\big|\delta|A|g(y)-\delta\mathbf{1}_A(y)\big|\le \delta|A|g(y)+\delta\mathbf{1}_{A}(y).$$
    Integrating the above shows that $\|g-g_{\delta}\|_{L^1{\mathbb{R}^q}}\le2\delta|A|$ and, correspondingly,
    $$\|f-g_\delta\|_{L^{1}(\mathbb{R}^q)}\ge
        \|f-g\|_{L^{1}(\mathbb{R}^q)}
        -
        2\delta|A|.$$
    Rearranging the above and using the fact that $0\le a\le b+c$ implies $\frac{1}{2}a^2\le b^2+c^2$, we get
    $$\frac{1}{2}\|f-g\|_{L^{1}(\mathbb{R}^q)}^2\le \|f-g_{\delta}\|_{L^1(\mathbb{R}^q)}^2+4\delta^2|A|^2\le 2\int_{\mathbb{R}^q}f(y)\log\left[\frac{f(y)}{\delta+g(y)}\right]dy+\delta|A|+4\delta^2|A|^2.$$
    The conclusion follows.
\end{proof}

\section{Technical details on the simulation studies}
\label{appendix:exp-details}

\subsection{Estimation of the AWD and related metrics}
\label{appendix:awd-details}

In the univariate test case, we notice that the AWD takes can be re-written more conveniently as
\begin{equation*}
        \operatorname{AWD}(f_{Y|X}, \hat f_{Y|X}) 
    = \int_{0}^1\left[ \int_0^1 
        \big| q_{Y|X=x}(\tau \mid x) - \hat q_{Y|X=x}(\tau \mid x) \big| \, \mathrm{d}\tau
    \right]dx.
\end{equation*}
where we explicitly leveraged the fact that $X\sim\mathcal{U}(0,1)$. We then approximate the above using Simpson's quadrature formula in both $x$ and $\tau$. To this end, we consider two uniform grids, $\tilde{x}_i=i/R_X$ with $R_X=1000$, and $\tau_j=j/R_{\tau}$ with $R_\tau=100$.

For QR models, the predicted conditional quantiles at the grid points, $q_{Y\mid X=\tilde{x}_i}(\tau_j\mid \tilde x_i)$, are computed with a simple forward pass through the model. For CPFN, KCDE and the true model, instead we approximate the conditional quantiles empirically by relying on conditional samples. Specifically, for each $i$, we generate $R_Y=1000$ \emph{iid} conditional samples from both $f_{Y\mid X=\tilde x_i}$ and $\hat f_{Y\mid X=\tilde x_i}$, resulting in $\{y_{i,k}\}_{k=1}^{R_Y}$ and $\{\hat{y}_{i,k}\}_{k=1}^{R_Y}$, respectively. Then, we use these samples to compute the empirical quantiles in the usual way.
We emphasize that the generation of conditional samples is straightforward for the true model and for CPFN (by design). For KCDE, instead, 
we rely on acceptance-rejection sampling: see Appendix~\ref{appendix:acceptance-rejection-details} for further details. As a side note, we mention that the estimates obtained here are also exploited to approximate the AQE, essentially using the fact that
$$\mathrm{AQE}(\tau_j)=\int_0^1 
        \big| q_{Y|X=x}(\tau_j \mid x) - \hat q_{Y|X=x}(\tau_j \mid x)|dx.$$
In the multivariate case, instead, we directly rely on conditional samples and exploit an alternative estimator of the $\mathcal{W}_1$-distance to approximate the AWD. Specifically, just as before, we first generate $R_X$ independent samples of $X$, denoted $\tilde x_1, \dots, \tilde x_{R_X}$. For each $\tilde x_i$, we then generate $R_Y$ conditional samples from both $f(y \mid \tilde x_i)$ and $\hat f(y \mid \tilde x_i)$, yielding $\{y_{i,j}\}_{j=1}^{R_Y}$ and $\{\hat y_{i,j}\}_{j=1}^{R_Y}$, respectively. Then, we consider an approximation of the form
\begin{equation*}
\operatorname{AWD}\approx\frac{1}{R_X}\sum_{i=1}^{R_X}\tilde{\mathcal{W}}_1(\{y_{i,j}\}_{j=1}^{R_Y}, \{\hat y_{i,j}\}_{j=1}^{R_Y}),
\end{equation*}
where $\tilde{\mathcal{W}}_1(P,Q)$ denotes the 1-Wasserstein distance between two clouds of points $P$ and $Q$ ---equivalent to the $\mathcal{W}_1$-distance between their corresponding empirical distributions. From a practical standpoint, we compute $\tilde{\mathcal{W}}_1$ using \texttt{emd2} function implemented in the \texttt{pot} package for Optimal Transport \citep{flamary2021pot}.
\\\\
We emphasize that, both in the unviariate and multivariate case, this procedure is done \emph{after} the model has been fit to the data. In particular, all samples drawn at this stage are generated independently of the training data.

\subsection{KCDE sampling via acceptance-rejection
}
\label{appendix:acceptance-rejection-details}
To generate conditional samples from the KCDE-based estimator $\hat f_{Y|X}$ in the multivariate test case, we employ an acceptance–rejection scheme. First, we define a uniform proposal density $g(y)$ over a rectangular region
$$
  \mathcal{B} = [y_{\min,1}, y_{\max,1}] \times [y_{\min,2}, y_{\max,2}],
$$
which encompasses all training responses and is expanded by three empirical standard deviations along each coordinate axis, ensuring that $\mathcal{B}$ covers essentially all of the probability mass of $\hat f_{Y|X}$. We then set the bounding constant $M$ as
$$
  M = 1.1 \, \frac{f_{\max}}{g(y)}, 
  \qquad 
  f_{\max} = \sup_{x, y} \hat f_{Y|X}(y \mid x),
$$
where $f_{\max}$ is approximated numerically over a fine grid, guaranteeing that $\hat f_{Y|X}(y \mid x) \le M g(y)$ for all $y \in \mathcal{B}$. For each test covariate $\tilde x_i$, we then draw a candidate point $y \sim g$ and a random, independent, threshold $U \sim \mathcal{U}[0,1]$. If
$$
  U < \frac{\hat f_{Y|X}(y \mid x_i)}{M g(y)},
$$
we accept the candidate and we include it as part of the conditional sample for $\tilde x_i$. We repeat this procedure iteratively, each time with an independent trial, until we reach the desired sample size.

\end{document}